\documentclass[twocolumn,10pt]{article}

\usepackage[utf8]{inputenc}
\usepackage{color}
\usepackage{pbox}
\usepackage{booktabs}
\usepackage{amsmath,amssymb}
\usepackage{amsthm}
\usepackage{verbatim} 
\usepackage{caption}
\usepackage{subcaption}
\usepackage{xcolor, soul}
\sethlcolor{yellow}

\usepackage{hyperref}
\usepackage{pdfpages}
\usepackage{balance}

\usepackage{bm}
\usepackage{graphicx}
\usepackage[absolute,overlay]{textpos}

\newcommand{\squishlist}{
	\begin{list}{$\bullet$}
		{ \setlength{\itemsep}{2pt}    \setlength{\parsep}{0pt}
			\setlength{\topsep}{5pt}     \setlength{\partopsep}{0pt}
			\setlength{\leftmargin}{1.35em} \setlength{\labelwidth}{1em}
			\setlength{\labelsep}{0.5em} } }
	
	\newcommand{\squishend}{
\end{list}  }

\makeatletter
\renewcommand{\fnum@figure}{\textbf{Fig.~\thefigure}}
\makeatother

\DeclareMathOperator{\E}{\mathbb{E}}
\newcommand*\diff{\mathop{}\!\mathrm{d}} 

\newtheorem{theorem}{Theorem}
\newtheorem{corollary}{Corollary}
\newtheorem{lemma}{Lemma}
\newtheorem{proposition}{Proposition}
\newtheorem{definition}{Definition}

\usepackage{hyperref}
\usepackage{mathtools}

\usepackage{geometry}
\newcount\Comments  
\usepackage{color}
\definecolor{darkgreen}{rgb}{0,0.5,0}
\definecolor{purple}{rgb}{1,0,1}
\newcommand{\kibitz}[2]{\ifnum\Comments=1\textcolor{#1}{#2}\fi}

\usepackage{helvet}
\usepackage[T1]{fontenc}

\usepackage[superscript]{cite}

\usepackage{abstract}
\renewcommand{\abstractname}{}    

\usepackage[subtle]{savetrees}

\begin{document}
	\begin{textblock*}{20cm}(1cm,1cm)
		\textcolor{red}{Preprint accepted by Nature Communications Engieering \url{https://www.nature.com/commseng/}}
	\end{textblock*}
	
	\title{Bayesian Learning for the Robust Verification of Autonomous Robots}

	\author{Xingyu Zhao\thanks{Warwick Manufacturing Group, University of Warwick, Coventry, UK.} ,
		Simos Gerasimou\thanks{Department of Computer Science, University of York, York, UK.} , 
		Radu Calinescu\thanks{Department of Computer Science and Assuring Autonomy International Programme, University of York, York, UK.} ,
		Calum Imrie\footnotemark[3] ,
		Valentin Robu\thanks{Intelligent and Autonomous Systems Group, Centrum Wiskunde \& Informatica, Amsterdam, NL.},
		\thanks{Electrical Engineering Department, Eindhoven University of Technology, Eindhoven, NL}
		and 
		David Flynn\thanks{James Watt School of Engineering, University of Glasgow, Glasgow, UK}
	}
	\date{}
	
	\twocolumn[
	\begin{@twocolumnfalse}
		\maketitle
		\begin{abstract}
			
			\textbf{
				\noindent
				Abstract -- Autonomous robots used in infrastructure inspection, space exploration and other critical missions operate in highly dynamic environments. As such, they must continually verify their ability to complete the tasks associated with these missions safely and effectively. Here we present a Bayesian learning framework that enables this runtime verification of autonomous robots. The framework uses prior knowledge and observations of the verified robot to learn expected ranges for the occurrence rates of regular and singular (e.g., catastrophic failure) events. Interval continuous-time Markov models defined using these ranges are then analysed to obtain expected intervals of variation for system properties such as mission duration and success probability. We apply the framework to an autonomous robotic mission for underwater infrastructure inspection and repair. The formal proofs and experiments presented in the paper show that our framework produces results that reflect the uncertainty intrinsic to many real-world systems, enabling the robust verification of their quantitative properties under parametric uncertainty. 
			}
		\end{abstract}
		
		\vspace*{3mm}
	\end{@twocolumnfalse}
	]
	\saythanks
	
	\noindent
	
	\section{Introduction}
	
	Mobile robots are increasingly used to perform critical missions in extreme environments, which are 
	inaccessible or hazardous to humans.~\cite{japan-strategy-2015,EU-Robotics-2020,UK-RAS-2014,christensen2021roadmap}
	These missions range from the inspection 
	and 
	maintenance of offshore wind-turbine mooring chains and high-voltage cables to nuclear reactor repair and deep-space exploration.~\cite{richardson2017robotic,UK-RAS-space-2018}
	
	Using robots for such missions poses major challenges.~\cite{lane_new_2016,EU-Robotics-2020} First and foremost, the robots need to operate with high levels of autonomy, as in these harsh environments their interaction and communication with human operators is severely restricted. Additionally, they frequently need to make complex mission-critical decisions, with errors endangering not just the robot---itself an expensive asset, but also the important system or environment being inspected, repaired or explored. Last but not least, they need to cope with the considerable uncertainty associated with these missions, which often comprise one-off tasks or are carried out in settings not encountered before. 
	
	Addressing these major challenges is the focus of intense research worldwide. In the UK alone, a recent \pounds 44.5M research programme has 
	tackled technical and certification challenges associated with the use of robotics and AI in the extreme environments encountered in offshore energy (\url{https://orcahub.org}), space exploration (\url{https://www.fairspacehub.org}), nuclear infrastructure (\url{https://rainhub.org.uk}), and management of nuclear waste (\url{https://www.ncnr.org.uk}). This research has initiated a step change in the assurance and certification of autonomous robots---not least through the emergence of new concepts such as dynamic assurance~\cite{calinescu2017engineering} and self-certification~\cite{robu_train_2018} for robotic systems. 
	
	Dynamic assurance requires a robot to respond to failures, environmental changes and other disruptions not only by reconfiguring accordingly,~\cite{calinescu2012self} but also by producing new assurance evidence which guarantees that the reconfigured robot will continue to achieve its mission goals.~\cite{calinescu2017engineering} Self-certifying robots must continually verify their health and ability to complete missions in dynamic, risk-prone environments.~\cite{robu_train_2018} In line with the ``defence in depth'' safety engineering paradigm,~\cite{INSAG-10} this runtime verification has to be performed independently of the front-end planning and control engine of the robot. 
	
	Despite these advances, 
	current dynamic assurance and self-certi\-fi\-ca\-tion methods rely on quantitative verification techniques (e.g., probabilistic~\cite{Kwi07,katoen_probabilistic_2016} and statistical~\cite{legay_statistical_2010} model checking) that do not handle well the parametric uncertainty that autonomous robots encounter in extreme environments. Indeed, quantitative verification  operates with stochastic models that demand single-point estimates of uncertain parameters such as task execution and failure rates. These estimates capture neither epistemic nor aleatory parametric uncertainty. As such, they are affected by arbitrary estimation errors which---because stochastic models are often nonlinear---can be amplified in the verification process~\cite{calinescu_2016_formal}, and may lead to invalid robot reconfiguration decisions, dynamic assurance and self-certification.
	
	In this paper, we present a robust quantitative verification framework that employs Bayesian learning techniques to overcome this limitation. Our framework requires only partial and limited prior knowledge about the verified robotic system, and exploits its runtime observations (or lack thereof) to learn ranges of values for the system parameters. These parameter ranges are then used to compute  the quantitative properties that underpin the robot's decision making (e.g., probability of mission success, and expected energy usage) as intervals that---unique to our framework---capture the parametric uncertainty of the mission. 
	Our framework is underpinned by probabilistic model checking, a technique that is broadly used to assess quantitative properties, e.g., reliability, performance and energy cost of systems exhibiting stochastic behaviour. Such systems include 
	autonomous robots from numerous domains{\cite{kwiatkowska_probabilistic_2022}}, e.g., mobile service robots{\cite{lacerda2019probabilistic}}, spacecraft{\cite{nardone2016probabilistic}}, drones{\cite{fraser2020collaborative}} and robotic swarms{\cite{liu2010modeling}}. While we present a case study involving an autonomous underwater vehicle (AUV), the generalisability of our approach stems from the broad adoption of probabilistic model checking for the modelling and verification of this wide range of autonomous robots. As such, we anticipate that our results are applicable to autonomous agents across all these domains.
	
	We start by introducing our robust verification framework, which comprises Bayesian techniques for learning the occurrence rates of both singular events (e.g., catastrophic failures and completion of one-off tasks) and events observed regularly during system operation. Next, we describe the use of the framework for an offshore wind turbine inspection and maintenance robotic mission. Finally, we discuss the framework in the context of related work, and we suggest directions for further research.
	
	\section{Robust Bayesian verification framework}
	
	\subsection{Quantitative verification}
	Quantitative verification is a mathematically based technique for analysing the correctness, reliability, performance and other key properties of systems with stochastic behaviour.~\cite{Baier-etal-2003,Kwiatkowska2007:SFM} The technique captures this behaviour into {Markov models}, formalises the properties of interest as {probabilistic temporal logic} formulae over these models, and employs efficient algorithms for their analysis. Examples of such properties include the probability of mission failure for an autonomous robot, and the expected battery energy required to complete a robotic mission.
	
	In this paper, we focus on the quantitative verification of {continuous-time Markov chains} (CMTCs). CTMCs are Markov models for continuous-time stochastic processes over countable state spaces comprising (i)~a finite set of {states} corresponding to real-world states of the system that are relevant for the analysed properties; and (ii)~the {rates of transition} between these states. 
	We use the following definition adapted from the probabilistic model checking literature~\cite{Baier-etal-2003,Kwiatkowska2007:SFM}.
	
	\begin{definition}
		\label{def:CTMC}
		A continuous-time Markov chain is a tuple
		\begin{equation}
		\mathcal{M}\!=\!(S,s_0,\mathbf{R}),
		\label{eq:CTMC} 
		\end{equation}
		where $S$ is a finite set of states, $s_0\!\in\! S$ is the initial state, and $\mathbf{R}\!:\! S\!\times\! S\!\to\! \mathbb{R}$ is a transition rate matrix such that the probability that the CTMC will leave state $s_i \!\in\! S$ within $t\!>\!0$ time units is  $1\!-\!e^{-t\cdot \sum_{s_k\!\in\! S\setminus\! \{s_i\}} \mathbf{R}(s_i,s_k)}$ and the probability that the new state is $s_j\in S\setminus\{s_i\}$ is $p_{ij}=$ $\mathbf{R}(s_i,s_j)\; /\;\sum_{s_k\in S\setminus\{s_i\}}\;\! \mathbf{R}(s_i,s_k)$.
	\end{definition}
	
	The range of properties that can be verified using CTMCs can be extended by annotating the states and transitions with non-negative quantities called {rewards}.
	
	\begin{definition}
		\label{def:rewards}
		A reward structure over a CTMC $\mathcal{M}\!=\!(S,s_0,$ $\mathbf{R})$ is a pair of functions $(\underline{\rho}, \bm{\iota})$ such that 
		$\underline{\rho} : S\to \mathbb{R}_{\geq 0}$ is a {state reward function} (a vector), and 
		$\bm{\iota} : S\times S\to \mathbb{R}_{\geq 0}$ is a {transition reward function} (a matrix).
	\end{definition}
	
	CTMCs support the verification of quantitative properties expressed in continuous stochastic logic (CSL)~\cite{aziz1996} extended with rewards~\cite{Kwiatkowska2007:SFM}. 
	
	\begin{definition}
		Given a set of atomic propositions $\mathit{AP}$, $a\!\in\!\mathit{AP}$, $p\! \in\! [0,1]$, $I\subseteq\mathbb{R}_{\geq 0}$, $r,t\!\in\!\mathbb{R}_{\geq 0}$ and $\bowtie\,\in\! \{\geq, >, <, \leq \}$, a CSL formula $\Phi$ is defined by the grammar:
		\begin{equation*}
		\begin{array}{ll}
		\!\!\!\Phi\; ::= &\hspace*{-2.3mm} true\; |\;  a\; |\; \Phi \wedge \Phi\; |\; \neg\Phi\; |\; P_{\bowtie p}[X\, \Phi]\; |\; P_{\bowtie p}[\Phi\, U^{I}\, \Phi]\; |\; \mathcal{S}_{\bowtie p}[\Phi]\; | \\
		&\hspace*{-2.3mm} R_{\bowtie r}[I^{=t}]\; | \; R_{\bowtie r}[C^{\leq t}]\; | \; R_{\bowtie r}[F\, \Phi]\; | \; R_{\bowtie r}[S].
		\end{array}
		\end{equation*}
	\end{definition}
	
	\noindent
	Given a CTMC $\mathcal{M}\!=\!(S,s_0,$ $\mathbf{R})$ with states {labelled} with atomic propositions from $AP$ by a function $L\! :\! S\! \to\! 2^{AP}$, and a reward structure $(\underline{\rho}, \bm{\iota})$ over $\mathcal{M}$, the CSL semantics is defined with a satisfaction relation $\models$ over the states and {paths} (i.e., feasible sequences of successive states) of $\mathcal{M}$~\cite{Baier-etal-2003}. 
	The notation $s\models \Phi$ means ``$\Phi$ is satisfied in state $s$''. For any state $s\!\in\! S$, we have: 
	
	\squishlist
	\item $s\models true$, 
	$s \models a$ iff $a\in L(s)$, 
	$s \models \neg \Phi$ iff $\neg (s\models \Phi)$, and 
	$s\models \Phi_1 \wedge \Phi_2$ iff $s\models \Phi_1$ and $s\models \Phi_2$;
	\item $s\models \mathcal{P}_{\bowtie p} [X\, \Phi]$ iff the probability $x$ that $\Phi$ holds in the state  following $s$ satisfies $x \bowtie p$ ({probabilistic next} formula); 
	\item $s\models \mathcal{P}_{\bowtie p} [\Phi_1\, U^{I}\, \Phi_2]$ iff, across all paths starting at $s$, the probability $x$ of going through only states where $\Phi_1$ holds until reaching a state where $\Phi_2$ holds at a time $t\in I$ satisfies $x\bowtie p$  ({probabilistic until} formula);
	\item $s\models \mathcal{S}_{\bowtie p}[\Phi]$ iff, having started in state $s$, the probability $x$ of $\mathcal{M}$ reaching a state where $\Phi$ holds {in the long run} satisfies $x \bowtie p$  ({probabilistic steady-state} formula);
	\item the {instantaneous ($R_{\bowtie r}[I^{=t}]$), cumulative ($R_{\bowtie r}[C^{\leq t}]$), future-state ($R_{\bowtie r}[F\, \Phi]$)} and {steady-state ($R_{\bowtie r}[S]$) reward} formulae hold iff, having started in state $s$, the expected reward $x$ at time instant $t$, cumulated up to time $t$, cumulated until reaching a state where $\Phi$ holds, and achieved at steady state, respectively, satisfies $x\bowtie r$.
	\squishend
	
	Probabilistic model checkers such as PRISM~\cite{kwiatkowska_prism_2011} and Storm~\cite{storm2017} use efficient analysis techniques to compute the actual probabilities and expected rewards associated with probabilistic and reward formulae, respectively. The formulae are then verified by comparing the computed values to the bounds $p$ and $r$. Furthermore, the extended CSL syntax $P_{=?}[X\, \Phi]$, $P_{=?}[\Phi_1\, U^{I}\, \Phi_2]$, $R_{=?}[I^{=t}]$, etc. can be used to obtain these values from the model checkers. 
	
	While the transition rates of the CTMCs verified in this way must be known and constant, advanced quantitative verification techniques~\cite{brim-etal-2013} support the analysis of CTMCs whose transition rates are specified as {intervals}. 
	The technique has been used to synthesise CTMCs corresponding to process configurations and system designs that satisfy quantitative constraints and optimisation criteria~\cite{calinescu_efficient_2018,ceska_prism_psy_2016,calinescu2017rodes}, under the assumption that these bounded intervals are available. 
	Here we introduce a Bayesian framework for computing these intervals in ways that reflect the parametric uncertainty of real-world systems such as autonomous robots.

	\subsection{Bayesian learning of CTMC transition rates} 
	Given two states $s_i$ and $s_j$ of a CTMC such that transitions from $s_i$ to $s_j$ are possible and occur with rate $\lambda$, each transition from $s_i$ to $s_j$ is independent of how state $s_i$ was reached (the Markov property). Furthermore, the time spent in state $s_i$ before a transition to $s_j$ is modelled by a homogeneous Poisson process of rate $\lambda$. Accordingly, the likelihood that `data' collected by observing the CTMC shows $n$ such transitions occurring within a combined time $t$ spent in state $s_i$ is given by the conditional probability:
	\begin{equation}
	\label{eq_likelihood_poisson_process}
	l(\lambda) = \mathit{Pr}\left(\textmd{data} \mid \lambda\right)=\frac{\left(\lambda t\right)^{n}}{n!}e^{-\lambda t}
	\end{equation}
	
	In practice, the rate $\lambda$ is typically unknown, but prior beliefs about its value are available (e.g., from domain experts or from past missions performed by the system modelled by the CTMC) in the form of a probability (density or mass) function $f(\lambda)$. In this common scenario, the Bayes Theorem can be used to derive a {posterior probability function} that combines the likelihood $l(\lambda)$ and the prior $f(\lambda)$ into a better estimate for $\lambda$ at time $t$:
	\begin{equation}
	f(\lambda\mid \textmd{data}) = \frac{l(\lambda)f(\lambda)}{\int_{0}^{\infty} l(\lambda)f(\lambda) \diff \lambda}
	\label{eq_bayes}
	\end{equation}
	where the Lebesgue-Stieltjes integral from the denominator is introduced to ensure that $f(\lambda\mid \textmd{data})$ is a probability function. We note, we use Lebesgue-Stieltjes integration to cover in a compact way both continuous and discrete prior distributions $f(\lambda)$, as these integrals naturally reduce to sums for discrete distributions.
	We calculate the posterior estimate for the rate $\lambda$ at time $t$ as the expectation of~\eqref{eq_bayes}:
	\begin{equation}
	\label{eq_posterior_mean_lambda}
	\lambda^{(t)}=\E\left[\Lambda \mid \text{data}\right]
	=\frac{\int_{0}^{\infty} \lambda l(\lambda)f(\lambda) \diff \lambda}{\int_{0}^{\infty} l(\lambda)f(\lambda) \diff \lambda}\;.
	\end{equation}
	where we use capital letters for random variables and lower case for their realisations.

	\subsection{Interval Bayesian inference for singular events} In the autonomous-robot missions considered in our paper, certain events are extremely rare, and treated as {unique} from a modelling viewpoint. 
	These events include major failures (after each of which the system is modified to remove or mitigate the cause of the failure), and the successful completion of difficult one-off tasks. Using Bayesian inference to estimate the CTMC transition rates associated with such events is challenging because, with no  observations of these events, the posterior estimate is highly sensitive to the choice of a suitable prior distribution. Furthermore, only limited domain knowledge is often available to select and justify a prior distribution for these singular events. 
	
	To address this challenge, we develop a {Bayesian inference using partial priors} (BIPP) estimator that requires only \textit{limited, partial prior knowledge} instead of the complete prior distribution typically needed for Bayesian inference. For one-off events, such knowledge is both more likely to be available and easier to justify. BIPP provides bounded posterior estimates that are robust in the sense that the ground truth rate values are within the estimated intervals.
	
	To derive the BIPP estimator, we note that for one-off events the likelihood~(\ref{eq_likelihood_poisson_process}) becomes 
	%
	%
	\begin{equation}
	l(\lambda)=Pr(\text{data}|\lambda)=e^{-\lambda \cdot t}
	\label{eq_likelihood_f}
	\end{equation}
	because $n=0$. 
	Instead of a prior distribution $f(\lambda)$ (required to compute the posterior expectation~(\ref{eq_posterior_mean_lambda})), we assume that we only have limited partial knowledge consisting of $m\geq 2$ confidence bounds on $f(\lambda)$:
	\begin{equation}
	\label{eq_prior_constraints}
	Pr(\epsilon_{i-1} < \lambda \leq \epsilon_i)=\theta_i 
	\end{equation}
	where $1\leq i \leq m$, $\theta_i>0$, and $\sum_{i=1}^{m} \theta_i=1$. The use of such bounds is a common practice for safety-critical systems. As an example, the IEC61508 safety standard \cite{iec_61508_2010} defines {safety integrity levels} (SILs) for the critical functions of a system based on the bounds for their probability of failure on demand ($\mathit{pfd}$): $\mathit{pfd}$ between $10^{-2}$ and $10^{-1}$ corresponds to SIL~1, $\mathit{pfd}$ between $10^{-3}$ and $10^{-2}$ corresponds to SIL~2, etc.; and testing can be used to estimate the probabilities that a critical function has different SILs. We note that $Pr(\lambda\geq \epsilon_0)=Pr(\lambda\leq \epsilon_m)=1$ and that, when no specific information is available, we can use $\epsilon_0=0$ and $\epsilon_m=+\infty$.
	
	
	The partial knowledge encoded by the constraints~(\ref{eq_prior_constraints}) is far from a complete prior distribution: an infinite number of distributions $f(\lambda)$ satisfy these constraints, and the result below provides bounds for the estimate rate~\eqref{eq_posterior_mean_lambda} across these distributions.
	
	\begin{theorem}
		\label{theorem_BIPP_bounds}
		The set $S_\lambda$ of posterior estimate rates~\eqref{eq_posterior_mean_lambda} computed for all prior distributions $f(\lambda)$ that satisfy~\eqref{eq_prior_constraints} has an infinum $\lambda_l$ and a supremum $\lambda_u$ given by:
		\begin{align}
		\label{eq_reachable_BIPP_lower_bound}
		\!\!\!\!\!\lambda_l = \min \left\{ \frac{\sum_{i=1}^m \left[\epsilon_i l(\epsilon_i)(1-x_i)\theta_i+\epsilon_{i-1}l(\epsilon_{i-1})x_i\theta_i\right]}
		{\sum_{i=1}^{m} [l(\epsilon_{i})(1-x_i)\theta_i+l(\epsilon_{i-1})x_i\theta_i]} \right|
		\nonumber\\
		\forall 1\!\leq\! i\!\leq\! m . x_i\in[0,1] \biggr\}\,,
		\end{align}
		\begin{align}
		\!\!\!\!\lambda_u
		= \max\left\{ \biggl. \frac{\sum_{i=1}^{m} \lambda_{i}l(\lambda_{i})\theta_i}
		{\sum_{i=1}^{m} l(\lambda_{i})\theta_i} \biggr| \,\forall 1\!\leq\! i\!\leq\! m . \lambda_i\!\in\!(\epsilon_{i-1},\epsilon_i] \right\}\!.
		\label{eq_reachable_BIPP_upper_bound}
		\end{align}
	\end{theorem}
	
	\noindent
	We provide a formal proof of Theorem~{\ref{theorem_BIPP_bounds}} in
	Section~{\hyperlink{bippProof}{4.1}}.
	
	The values $\lambda_l$ and $\lambda_u$ can be computed using numerical optimisation software packages available, for instance, within widely used mathematical computing tools like MATLAB and Maple. For applications where computational resources are limited or the BIPP estimator is used online with tight deadlines, we provide closed-form estimator bounds for $m=3$ (with $m=2$ as a subcase).
	
	
	\begin{corollary}
		\label{corollary_closed_form_BIPP_3}
		When $m=3$, the bounds~\eqref{eq_reachable_BIPP_lower_bound}  and~\eqref{eq_reachable_BIPP_upper_bound} satisfy:
		\begin{align}
		\lambda_l
		\geq 
		\begin{cases} \frac{\epsilon_{1}l(\epsilon_{1})\theta_2}{\theta_1+l(\epsilon_{1})\theta_2}, & \text{if } \frac{\theta_2 (\epsilon_{1}-\epsilon_2)}{\theta_1}
		>\frac{\epsilon_{2}l(\epsilon_{2})- \epsilon_{1}l(\epsilon_{1}) }{l(\epsilon_{1})l(\epsilon_{2})}
		\\[1mm]
		\frac{\epsilon_{2}l(\epsilon_{2})\theta_2}{\theta_1+l(\epsilon_{2})\theta_2}, & \text{otherwise}\end{cases}
		\label{eq:closed-form-lower-bound}
		\end{align}	
		and
		\begin{align}
		\lambda_u < \begin{cases}
		\frac{\epsilon_1 l(\epsilon_1) \theta_1 +\epsilon_2 l(\epsilon_2) \theta_2+\frac{1}{t} l(\frac{1}{t})(1-\theta_1-\theta_2)}
		{ l(\epsilon_1) \theta_1 }, & \text{if } t < \frac{1}{\epsilon_{2}}
		\\[1mm]
		\frac{\epsilon_1 l(\epsilon_1) \theta_1 +\frac{1}{t} l(\frac{1}{t}) \theta_2+\epsilon_2 l(\epsilon_2)(1-\theta_1-\theta_2)}
		{ l(\epsilon_1) \theta_1 }, & \text{if }  \frac{1}{\epsilon_{2}} \leq t \leq \frac{1}{\epsilon_{1}}
		\\[1mm]
		\frac{\epsilon_1 l(\epsilon_1) (\theta_1+\theta_2) +\epsilon_2 l(\epsilon_2)(1-\theta_1-\theta_2)}
		{ l(\epsilon_1) \theta_1 }, & \text{otherwise} 
		\end{cases}
		\label{eq:closed-form-upper-bound}
		\end{align}
	\end{corollary}
	
	\begin{corollary}
		\label{corollary_closed_form_BIPP_2}
		Closed-form BIPP bounds for $m=2$ can be obtained
		by setting $\epsilon_2=\epsilon_1$ and $\theta_2=0$ in \eqref{eq:closed-form-lower-bound} and \eqref{eq:closed-form-upper-bound}.
	\end{corollary}

	\begin{figure*}
		\centering
		\includegraphics[width=\hsize]{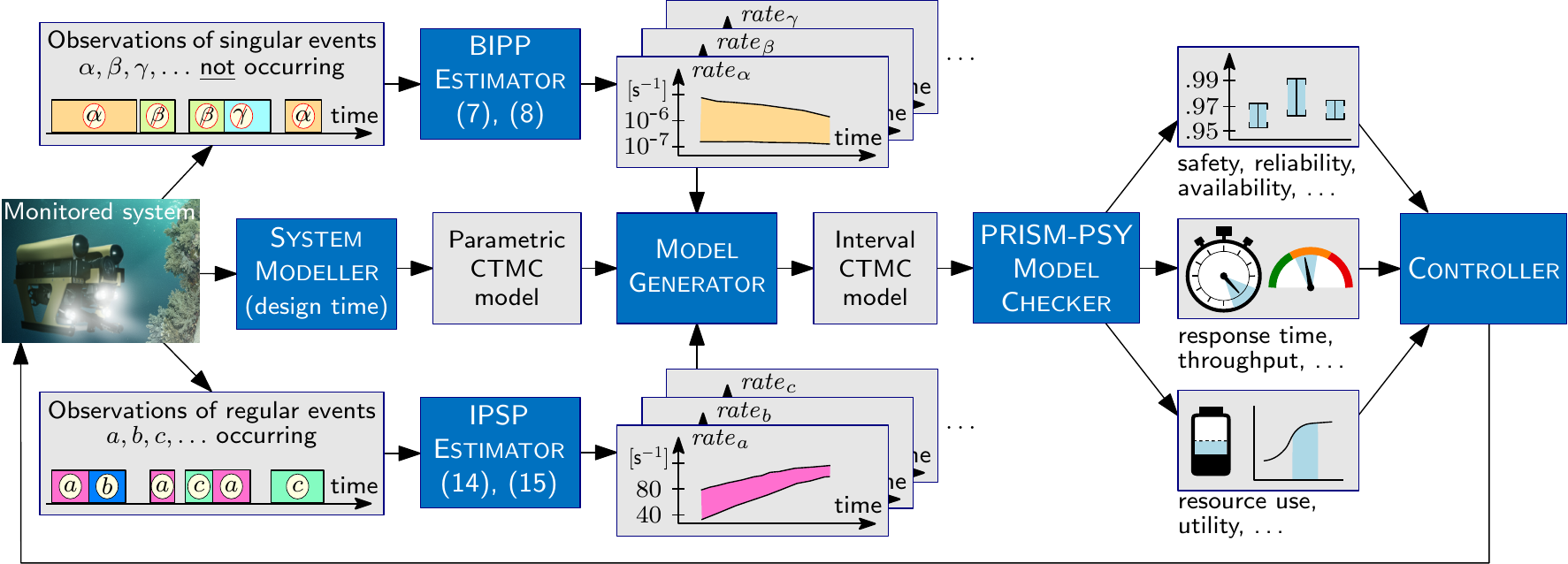}
		\caption{\textbf{Robust Bayesian verification framework.} The integration of Bayesian inference
			using partial priors (BIPP) and Bayesian inference using imprecise probability with sets of priors (IPSP)  
			with interval continuous-time Markov chain (CMTC) model checking supports the online robust quantitative verification and reconfiguration of autonomous systems under parametric uncertainty.}
		\label{fig:integration}
	\end{figure*}
	
	\subsection{Interval Bayesian inference for regular events}
	For CTMC transitions that correspond to regular events within the modelled system, we follow the common practice~\cite{bernardo_bayesian_1994} of using a Gamma prior distribution for each uncertain transition rate $\lambda$:
	\begin{equation}
	f(\lambda) = \Gamma[\lambda;\alpha,\beta]=\frac{\beta^\alpha}{(\alpha-1)!}\lambda^{\alpha-1}e^{-\beta\lambda}.
	\label{eq_gamma}
	\end{equation}
	The Gamma distribution is a frequently adopted conjugate prior distribution for the likelihood~\eqref{eq_likelihood_poisson_process} and, if the prior knowledge assumes an initial value $\lambda^{(0)}$ for the transition rate, the parameters $\alpha>0$ and $\beta>0$ must satisfy 
	\begin{equation}
	\E(\Gamma[\lambda;\alpha,\beta])=\frac{\alpha}{\beta}=\lambda^{(0)}.
	\end{equation}
	The posterior value $\lambda^{(t)}$ for the transition rate after observing $n$ transitions within $t$ time units is then obtained by using the prior~\eqref{eq_gamma} in the expectation~\eqref{eq_posterior_mean_lambda}, as in the following derivation adapted from classical Bayesian theory~\cite{bernardo_bayesian_1994}:
	\begin{multline}
	\lambda^{(t)} = \frac{\int_0^\infty \lambda\left(\frac{(\lambda t)^n}{n!}e^{-\lambda t}\right)\left(\frac{\beta}{(\alpha-1)!}\lambda^{\alpha-1}e^{-\beta\lambda}\right)\diff \lambda}{\int_0^\infty \left(\frac{(\lambda t)^n}{n!}e^{-\lambda t}\right)\left(\frac{\beta}{(\alpha-1)!}\lambda^{\alpha-1}e^{-\beta\lambda}\right)\diff \lambda}\\
	= \frac{\int_0^\infty \lambda^{n+\alpha}e^{-\lambda(t+\beta)}\diff \lambda}{\int_0^\infty \lambda^{n+\alpha-1}e^{-\lambda(t+\beta)}\diff \lambda}
	= \frac{\int_0^\infty \lambda^{n+\alpha}\left(\frac{e^{-\lambda(t+\beta)}}{-(t+\beta)}\right)'\diff \lambda}{\int_0^\infty \lambda^{n+\alpha-1}e^{-\lambda(t+\beta)}\diff \lambda}\\
	= \frac{\left. \left(\lambda^{n+\alpha}\frac{e^{-\lambda(t+\beta)}}{-(t+\beta)}\right) \right|_0^\infty - \int_0^\infty (n+\alpha)\lambda^{n+\alpha-1}\frac{e^{-\lambda(t+\beta)}}{-(t+\beta)}\diff \lambda}{\int_0^\infty \lambda^{n+\alpha-1}e^{-\lambda(t+\beta)}\diff \lambda}\\
	= \frac{0+\frac{n+\alpha}{t+\beta}\int_0^\infty \lambda^{n+\alpha-1}e^{-\lambda(t+\beta)}\diff \lambda}{\int_0^\infty \lambda^{n+\alpha-1}e^{-\lambda(t+\beta)}\diff \lambda}
	=\frac{n+\alpha}{t+\beta}
	=\frac{n+\beta\lambda^{(0)}}{t+\beta}\\
	=\frac{\beta}{t+\beta}\lambda^{(0)} +\frac{t}{t+\beta}\frac{n}{t}
	=\frac{t^{(0)}}{t+t^{(0)}}\lambda^{(0)} +\frac{t}{t+t^{(0)}}\frac{n}{t},
	\label{eq_posterior_rate_gamma}
	\end{multline}
	where $t^{(0)}=\beta$. This notation reflects the way in which the posterior rate $\lambda^{(t)}$ is computed as a weighted sum of the mean rate $\frac{n}{t}$ observed over a time period $t$, and of the prior $\lambda^{(0)}$ deemed as trustworthy as a mean rate calculated from observations over a time period $t^{(0)}$.
	When $t^{(0)}\ll t$ (either because we have low trust in the prior $\lambda^{(0)}$ and thus $t^{(0)}\simeq 0$, or because the system was observed for a time period $t$ that is much longer than $t^{(0)}$), the posterior~\eqref{eq_posterior_rate_gamma} reduces to the maximum likelihood estimator, i.e.\ $\lambda^{(t)} \simeq \frac{n}{t}$. In this scenario, the observations fully dominate the estimator~\eqref{eq_posterior_rate_gamma}, with no contribution from the prior.
	
	The selection of suitable values for the parameters $t^{(0)}$ and $\lambda^{(0)}$ of the traditional Bayesian estimator~\eqref{eq_posterior_rate_gamma} is very challenging. What constitutes a suitable choice often depends on unknown attributes of the environment, or several domain experts may each propose different values for these parameters. In line with recent advances in imprecise probabilistic modelling,
	\cite{krpelik2018imprecise,walter_bayesian_2017,walter_imprecision_2009} we address this challenge by defining a robust transition rate estimator for {Bayesian inference using imprecise probability with sets of priors} (IPSP).
	The IPSP estimator uses ranges $[\underline{t}^{(0)},\overline{t}\textrm{}^{(0)}]$ and $[\underline{\lambda}^{(0)},\overline{\lambda}\textrm{}^{(0)}]$ (corresponding to the  environmental uncertainty, or to input obtained from multiple domain experts) for the two parameters instead of point values. 
	
	The following theorem quantifies the uncertainty that the use of parameter ranges for $t^{(0)}$ and $\lambda^{(0)}$ induces on the posterior rate~\eqref{eq_posterior_rate_gamma}. 
	%
	%
	%
	%
	%
	This theorem specialises and builds on generalised Bayesian inference results~\cite{walter_imprecision_2009} that we adapt for the estimation of CTMC transition rates.

	\begin{theorem}
		\label{theorem_CTMC_IPSP}
		Given uncertain prior parameters $t^{(0)}\in[\underline{t}^{(0)},$ $\overline{t}\textrm{}^{(0)}]$ and $\lambda^{(0)}\in[\underline{\lambda}^{(0)},\overline{\lambda}\textrm{}^{(0)}]$, the posterior rate $\lambda^{(t)}$ from~\eqref{eq_posterior_rate_gamma} can range in the interval  $[\underline{\lambda}^{(t)},\overline{\lambda}\textrm{}^{(t)}]$, where:
		%
		\begin{equation}
		\underline{\lambda}^{(t)}=\begin{cases} \frac{\overline{t}\text{}^{(0)}\underline{\lambda}^{(0)}+n}{\overline{t}\text{}^{(0)}+t}, & \text{if } \frac{n}{t} \geq \underline{\lambda}^{(0)}  
		\\[2mm] 
		\frac{\underline{t}^{(0)}\underline{\lambda}^{(0)}+n}{\underline{t}^{(0)}+t}, & \text{otherwise} 
		\end{cases}
		\\[1mm]
		\label{eq_post_lower_bound_rij}
		\end{equation}
		and
		\begin{equation}
		\overline{\lambda}^{(t)}=\begin{cases} \frac{\overline{t}\text{}^{(0)}\overline{\lambda}\text{}^{(0)}+n}{\overline{t}\text{}^{(0)}+t}, & \text{if } \frac{n}{t} \leq \overline{\lambda}\text{}^{(0)}  
		\\[2mm]
		\frac{\underline{t}^{(0)}\overline{\lambda}\text{}^{(0)}+n}{\underline{t}^{(0)}+t}, & \text{otherwise} 
		\end{cases}.
		\label{eq_post_upper_bound_rij}
		\end{equation}
	\end{theorem}
	
	\noindent
	We provide a formal proof of Theorem~{\ref{theorem_CTMC_IPSP}} in 
	Section~{\hyperlink{ipspProof}{4.2}}.

	\subsection{Robust verification and adaptation} 
	
	Based on the methods defined earlier, we developed an end-to-end verification framework for the online computation of bounded intervals of CTMC properties. The verification framework integrates our BIPP and IPSP Bayesian interval estimators with interval CTMC model checking~\cite{brim-etal-2013}.  
	As shown in Fig.~\ref{fig:integration}, this involves devising a parametric CTMC model that captures the structural aspects of the system under verification through a \textsc{System Modeller}.
	This activity is typically performed once at design time, i.e., before the system is put into operation.
	By monitoring the system under verification, our framework enables observing both the occurrence of regular events and {the lack of singular events} during times when such events could have occurred (e.g., a catastrophic failure not happening when the system performs a dangerous operation). Our online \textsc{BIPP Estimator} and \textsc{IPSP Estimator} use these observations to calculate expected ranges for the rates of the monitored events, enabling a \textsc{Model Generator} to continually synthesise up-to-date interval CTMCs that model the evolving behaviour of the system. 
	
	The interval CTMCs, which are synthesised from the parametric CTMC model,
	are then continually verified by the \textsc{PRISM-PSY Model Checker}~\cite{ceska_prism_psy_2016}, to compute value intervals for key system properties. As shown in Fig.~\ref{fig:integration} and illustrated in the next section, these properties range from dependability (e.g., safety, reliability and availability) \cite{avizienis_basic_2004} and performance (e.g., response time and throughput) properties to resource use and system utility. Finally, changes in the value ranges of these properties may prompt the dynamic reconfiguration of the system by a \textsc{Controller} module responsible for ensuring that the system requirements are satisfied at all times.

	\section{Results: Robust verification of robotic mission}
	
	\begin{figure*}[t!]
		\centering
		\includegraphics[width=\hsize]{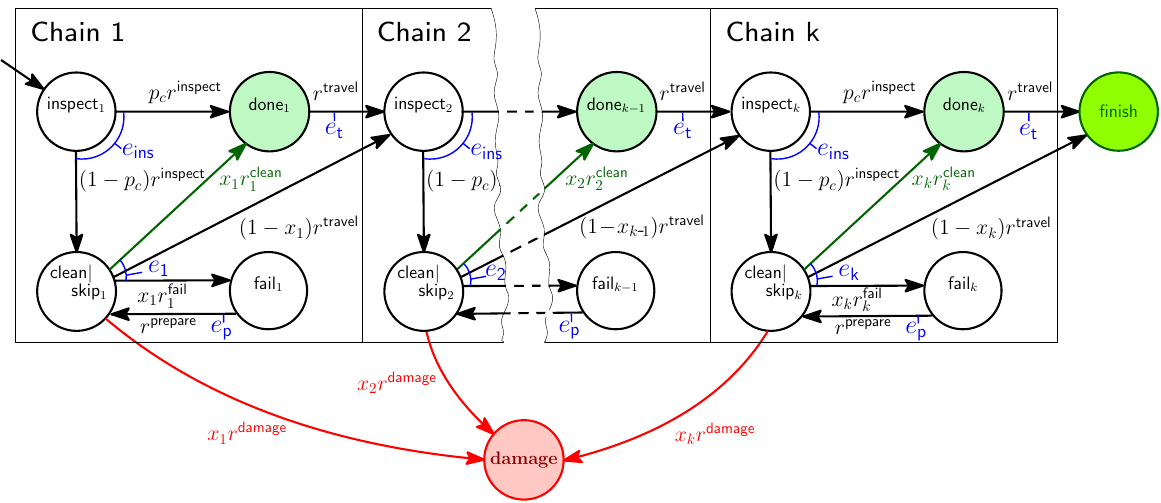}
		\caption{\textbf{Floating chain continuous-time Markov chain (CMTC) model. }CTMC of the floating chain cleaning and inspection mission, where $e_1$, $e_2$, \ldots, $e_k$ represent the mean energy required to clean chains $1$, $2$, \ldots, $k$, respectively. 
			The rate $r^{\mathsf{fail}}_i$ corresponds to a regular event and is therefore modelled using Bayesian inference using imprecise probability with sets of priors (IPSP) from {\eqref{eq_post_lower_bound_rij}} and {\eqref{eq_post_upper_bound_rij}}. 
			The rates $r_{i}^{\mathsf{clean}}$ and $r^{\mathsf{damage}}$ correspond to singular events and are thus modelled using Bayesian inference using partial priors (BIPP) from~{\eqref{eq_reachable_BIPP_lower_bound}} and {\eqref{eq_reachable_BIPP_upper_bound}}.}
		\label{fig:ctmc}
	\end{figure*}
	
	\begin{table*}[th!]
		\renewcommand{\arraystretch}{1.2}
		\centering
		\caption{System requirements for the AUV floating chain inspection and cleaning mission}
		
		\begin{small}
			\begin{tabular}{p{.2cm} p{10.8cm} p{5cm}} %
				\toprule
				\textbf{ID} 
				&\textbf{Informal description} 
				&{\raggedright\textbf{Formal specification$^\dagger$}}
				\\
				\midrule
				\textbf{R1}
				&The probability of mission failure  must not exceed 5\%
				&$P_{\leq 0.05} [F \textsf{ damage}]$\\
				\textbf{R2} 
				&The expected energy consumption must not exceed the remaining energy $E_\mathrm{left}$
				&$R^\mathrm{energy}_{\leq E_\mathrm{left}} [F \textsf{ finish}]$\\
				\textbf{R3}
				& Subject to R1 and R2 being met, maximise the number of cleaned chains
				&{\raggedright find 
					{argmax} $\sum_{i=1}^k x_i$ 
					such that $R1 \land R2$}\\
				\bottomrule
				\multicolumn{3}{l}{$^\dagger$expressed in rewards-extended continuous stochastic logic (see Methods section)}
			\end{tabular}
		\end{small}
		\label{table:reqs}	
	\end{table*}
	
	\subsection{Offshore infrastructure maintenance} We demonstrate how our online robust verification and reconfiguration framework can support an AUV to execute a structural health inspection and cleaning mission of the substructure of an offshore wind farm. 
	Similar scenarios for AUV use in remote, subsea environments have been described in other large-scale robotic demonstration projects, such as the PANDORA EU FP7 project~\cite{pandora2015}. 
	Compared to remotely operated vehicles that must be tethered with expensive oceanographic surface vessels run by specialised personnel, AUVs bring important advantages, including reduced environmental footprint (since no surface vessel consuming fuel is needed), reduced cognitive fatigue for the involved personnel, increased frequency of mission execution, and reduced operational and maintenance cost. 
	
	The offshore wind farm comprises multiple floating wind turbines, with each turbine being a buoyant foundation structure secured to the sea bed with floating chains tethered to anchors weighing several tons.  Wind farms with floating wind turbines offer increased wind exploitation (since they can be installed in deeper waters where  winds are stronger and more consistent), reduced installation costs (since there is no need to build solid foundations), and reduced impact on the visual and maritime life (since they are further from the shore)~\cite{myhr2014levelised}. 
	
	The AUV is deployed to collect data about the condition of $k \geq 1$ floating chains to enable the post-mission identification of problems that could
	affect the structural integrity of the asset (floating chain). 
	When the visual inspection of a chain is hindered due to accumulated biofouling or marine growth, the AUV can use its on-board high-pressure water jet to clean the chain and continue with the inspection task~\cite{pandora2015}. 
	
	The high degrees of \textit{aleatoric uncertainty} in navigation and the perception of the marine environment entail that the AUV might fail to clean a chain. This uncertainty originates from the dynamic conditions of the underwater medium that includes unexpected water perturbations coupled with difficulties in scene understanding due to reduced visibility and the need to operate close to the floating chains. When this occurs, the AUV can retry the cleaning task or skip the chain and move to the next.
	
	\subsection{Stochastic mission modelling}
	
	
	Fig.~\ref{fig:ctmc} shows the parametric CTMC model of the floating chain inspection and cleaning mission. 
	The AUV inspects the $i$-th chain with rate $r^\mathrm{inspect}$ and consumes energy $e_{ins}$. 
	The chain is clean with probability $p_{\mathsf{c}}$ and the AUV travels to the  next chain with rate $r^\mathrm{travel}$ consuming energy $e_t$, or the chain needs cleaning with probability $1-p_{\mathsf{c}}$.
	When the AUV attempts the cleaning ($x_i=1$), the task succeeds with chain-dependent rate $r_i^\mathrm{clean}$, causes catastrophic damage to the floating chain or itself with rate $r^\mathrm{damage}$ or fails with chain-dependent 
	rate $r^\mathrm{fail}_i$. If the cleaning fails, the AUV prepares to retry with known and fixed rate $r^\mathrm{prepare}$ requiring energy $e_p$, and it either retries cleaning ($x_i=1$) or skips the current chain and moves to chain $i\!+\!1$ ($x_i=0$). 
	After executing the tasks on the $k$-th chain, the AUV returns to its base and completes the mission. 
	
	Since the AUV can fail to clean the $i$-th chain with non-negligible probability and multiple times, this is a regular event whose  transition rate $r^\mathrm{fail}_i$ is modelled using the IPSP estimator from (7) and (8).
	In contrast, 
	the AUV is expected to not cause catastrophic damage but, with extremely low probability, may do so only once (after which the AUV and/or its mission are likely to be revised); thus, the corresponding transition rates $r^\mathrm{clean}_i$ and $r^\mathrm{damage}$ are modelled using the BIPP estimator from (14) and (15). The other transition rates, i.e., those for inspection ($r^\mathrm{inspect}$), travelling ($r^\mathrm{travel}$) and preparation ($r^\mathrm{prepare}$), are less influenced by the chain conditions and therefore assumed to be known, e.g., from previous trials and missions; hence, we fixed these transition rates.
	
	When cleaning is needed for the $i$-th chain, the AUV controller synthesises a plan by determining the control parameter $x_i \in \{0,1\}$  for all remaining chains $i, i+1, ...k$ so that the system requirements in Table~\ref{table:reqs} are satisfied.

	\subsection{Robust verification results}
	We demonstrate our solution for robust verification and adaptation using a mission in which the AUV was deployed  to inspect and, if needed, clean six chains placed in a hexagonal arrangement (Fig.~\ref{fig:uuv}). 
	We used $m=3$ and the BIPP estimator {\eqref{eq_reachable_BIPP_lower_bound}} and {\eqref{eq_reachable_BIPP_upper_bound}} for the transition rates $r^\mathrm{clean}_i$ and $r^\mathrm{damage}$, which correspond to singular events. For $r^\mathrm{clean}_i$, we used 
	$\epsilon_1=0.12 + \mathcal{U}(0,0.12)$, 
	$\theta_1=0.10 + \mathcal{U}(0,0.001)$, 
	$\epsilon_2=0.90 + \mathcal{U}(0,0.90)$,
	$\theta_2=0.85 + \mathcal{U}(0,0.0085)$, where $\mathcal{U}(x, y)$ denotes a continuous uniform distribution with $x$ and $y$ being its minimum and maximum values, respectively. 
	For $r^\mathrm{damage}$, we used
	$\epsilon_1=1e\textrm{-8} + \mathcal{U}(0,1e\textrm{-8})$, 
	$\theta_1=0.88 + \mathcal{U}(0,0.0088)$, 
	$\epsilon_2=1e\textrm{-7} + \mathcal{U}(0,1e\textrm{-7})$,
	$\theta_2=0.10 + \mathcal{U}(0,0.001)$.
	For $r^\mathrm{fail}_i$, we used $t^{(0)} = [10 + \mathcal{U}(0,10)]$ and 
	$\lambda ^{(0)} = [0.0163 + \mathcal{U}(0,0.00163)]$.
	During the mission execution, the AUV performs the model checking at every cleaning attempt so that runtime observations are incorporated into the decision making process entailing also 
	that the currently synthesised plan is not necessarily used at subsequent chains. 
	Hence, the AUV only needs to check system configurations where at least the current chain is to be cleaned, thus halving the number of configurations to be checked (since configurations with $x_i=0$ need not be checked). If all of these checks that consider $x_i=1$ fail to satisfy the requirements from Table~\ref{table:reqs}, then the AUV decides to skip the current chain and proceed to inspect and clean the next chain. 

	\begin{figure*}[h!]
		\centering
		\includegraphics[width=0.83\hsize]{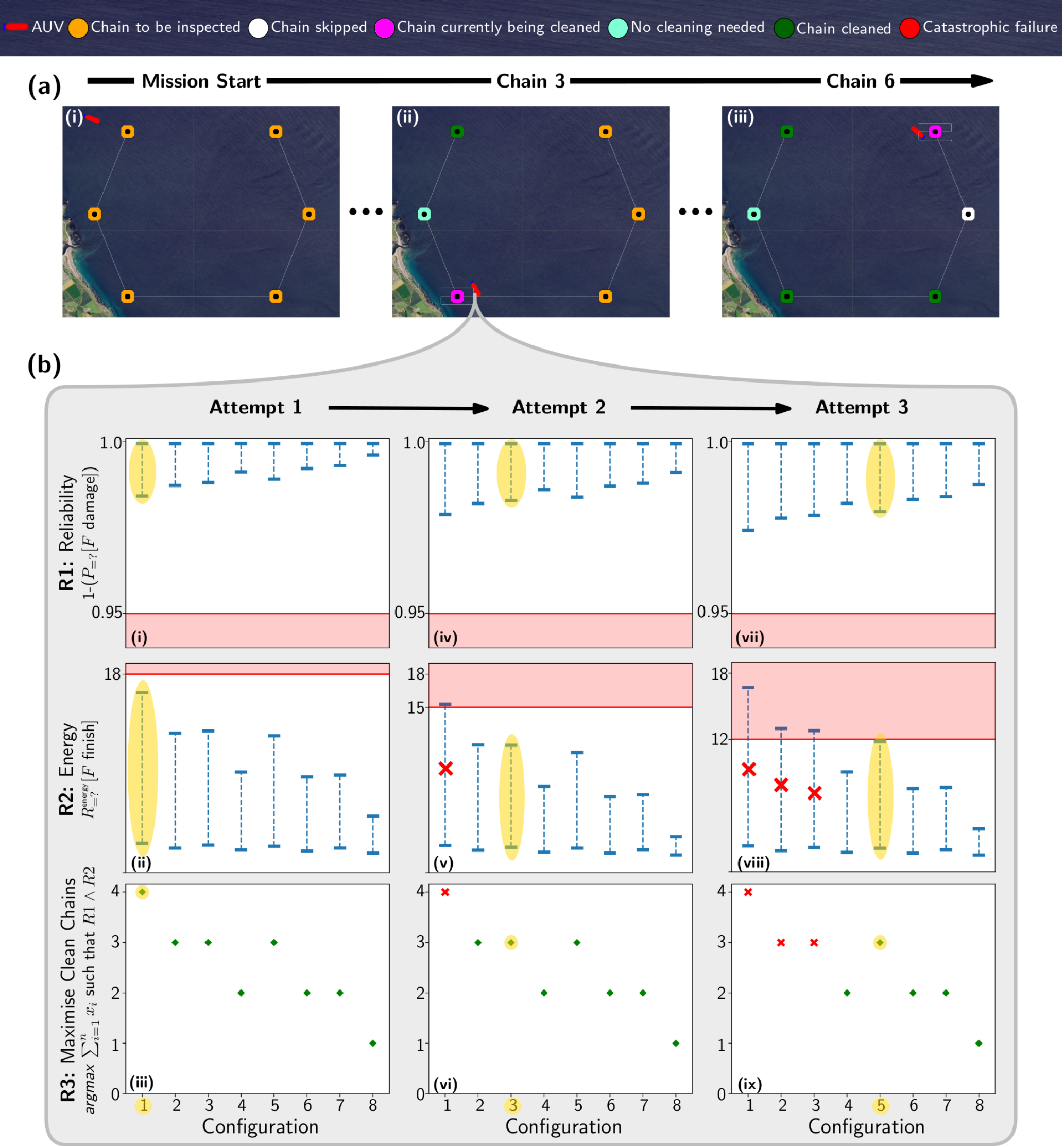}
		
		\caption{ \textbf{Demonstration of autonomous underwater
				vehicle (AUV) inspection and cleaning mission. } 
			\textbf{a} Simulated AUV mission involving the inspection of six wind farm chains and, if required, their cleaning. \textbf{i)} Start of mission; \textbf{ii)} cleaning chain 3; \textbf{iii)} cleaning final chain. At this point, the AUV cleaned three chains, skipped one, and one chain did not require cleaning. 
			\textbf{b} Plots of the outcome of the model checking carried out by the AUV at chain~3. Each row shows the configurations against the requirements. \textbf{(i-iii)} Results during the first attempt at cleaning chain 3. \textbf{(iv-vi)} Results during the second attempt at cleaning. \textbf{(vii-ix)} Results at the third and successful attempt at cleaning the chain. The configurations highlighted in yellow is the chosen configuration for the corresponding attempt.
		}
		\label{fig:uuv}
	\end{figure*}
	
	If a cleaning attempt at chain $i$ failed, 
	the AUV integrates this observation in \eqref{eq_post_lower_bound_rij}\eqref{eq_post_upper_bound_rij}, and performs model checking to determine whether to retry the cleaning or skip the chain. 
	Since the AUV has consumed energy for the failed cleaning attempt, the energy available is reduced accordingly, which in turn can reduce the number of possible system configurations that can be employed and need checking.
	The observation of a failed attempt reduces the lower bound for the reliability of cleaning $x_i$, and may result in a violation of the reliability requirement R1 (Table~\ref{table:reqs}), which may further reduce the number of feasible configurations.
	If the AUV fails to clean chain $i$ repeatedly, this lower bound will continue to decrease, potentially resulting in the AUV having no feasible configuration, and having to skip the current chain.
	Although skipping a chain overall decreases the risk of a catastrophic failure (as the number of cleaning attempts is reduced), leaving uncleaned chains will incur additional cost as a new inspection mission will need to be launched, e.g., using another AUV or human personnel.

	Fig.~\ref{fig:uuv} shows a simulated run of the AUV performing an inspection and cleaning mission (Fig.~\ref{fig:uuv}a).  At each chain that requires cleaning, the AUV decides whether to attempt to clean or skip the current chain.  
	Fig.~\ref{fig:uuv}b provides details of the probabilistic model checking carried out during the inspection and cleaning of chain~3 (Fig.~\ref{fig:uuv}a ii). 
	Overall, the AUV performed multiple attempts to clean chain~3, succeeding on the third attempt.
	
	The results of the model checking analyses for these attempts are shown in successive columns in Fig.~\ref{fig:uuv}b, while each row depicts the analysis of one of the requirements from Table~\ref{table:reqs}. 
	A system configuration is feasible if it satisfies requirements R1 --- the AUV will not encounter a catastrophic failure with a probability of at least 0.95, and R2 --- the expected energy consumption does not exceed the remaining AUV energy. 
	Lastly, if multiple configurations satisfy requirements R1 and R2, then the winner is the configuration that maximises the number of chains cleaned. 
	If there is still a tie, the configuration is chosen randomly from those that clean the most chains.
	
	In the AUV's first attempt at chain~3 (Fig.~\ref{fig:uuv}{b (i, ii, iii)}), all the configurations are feasible, so configuration 1 (highlighted, and corresponding to the highest number of chains cleaned) is selected. 
	This attempt fails, and a second assessment is made (Fig.~\ref{fig:uuv}{b (iv, v, vi)}). 
	This time, only system configurations 2--8 are feasible, and as configurations~2, 3, and~5 maximise R3, a configuration is chosen randomly from this subset (in this case, configuration~3). 
	This attempt also fails, and on the third attempt (Fig.~\ref{fig:uuv}{b (vii, viii, ix)}), only configurations 4--8 are feasible, with 5 maximising R3, and the AUV adopts this configuration and succeeds in cleaning the chain.
	
	In this AUV mission instance, the AUV controller is concerned with cleaning the maximum number of chains and ensuring the AUV returns safely.
	In other variants of our AUV mission, the system properties from requirements R1 and R2 could also be used to determine a winning configuration in the event of a tie between multiple feasible configurations. For example, it might be optimal for the AUV to consume minimal energy in this scenario. Thus, the energy consumption from requirement R2 can be used as a metric to choose a configuration as a tie-breaker.
	
	We also measured the overheads associated with executing the online verification process. 
	Figure~{\ref{fig:verification}} shows the computation overheads incurred by the RBV framework for executing the AUV-based mission. 
	The values comprising each boxplot have been collected over 10 independent runs. Each value denotes the time consumed for a single online robust quantitative verification and reconfiguration step when the AUV attempts to clean the indicated chain. 
	For instance, the boxplot associated with the `Chain 1' (`Chain 2') label on the x-axis signifies that the AUV attempts to clean chain 1  (chain 2) and corresponds to the time consumed by the RBV framework to analyse 64 (32) configurations. 
	Overall, the time overheads are reasonable for the purpose of this mission.
	Since the AUV has more configurations to analyse at the earlier stages of the mission (e.g., when inspecting chain 1), the results follow the anticipated exponential pattern. 
	The number of configurations decreases by half each time the AUV progresses further into the mission and moves to the next chain.  
	Another interesting observation is that the length of each boxplot is small, i.e., the lower and upper quartiles are very close, indicating that the RBV framework showcases a consistent behaviour in the time taken for its execution.

	The consumed time comprises (1) the time required to compute the posterior estimate bounds of the modelled transition rates,  $r_i^\textrm{clean}$, $r^{\textrm{fail}}$, $1 \leq i \leq k$, and  $r^{\textrm{damage}}$, using the BIPP and IPSP estimators; (2) the time required to compute the value intervals for requirements R1 and R2 using the probabilistic model checker \textsc{Prism- Psy}~{\cite{ceska_prism_psy_2016}}; and (3) the time needed to find the best configuration satisfying requirements R1 and R2, and maximising requirement R3. 
	Our empirical analysis provided evidence that the execution of the BIPP and IPSP estimators and the  selection of the best configuration have negligible overheads with almost all time incurred by \textsc{Prism-Psy}. 
	This outcome is not surprising and is aligned with the results reported in~{\cite{ceska_prism_psy_2016}} concerning the execution overheads of the model checker.
	
	The simulator used for the AUV mission, developed on top of the open-source MOOS-IvP middleware,~\cite{Benjamin2013MRO} and a video showing the execution of this AUV mission instance are available at~\url{http://github.com/gerasimou/RBV}.

	\begin{figure}[t]
		\centering
		\includegraphics[width=\linewidth]{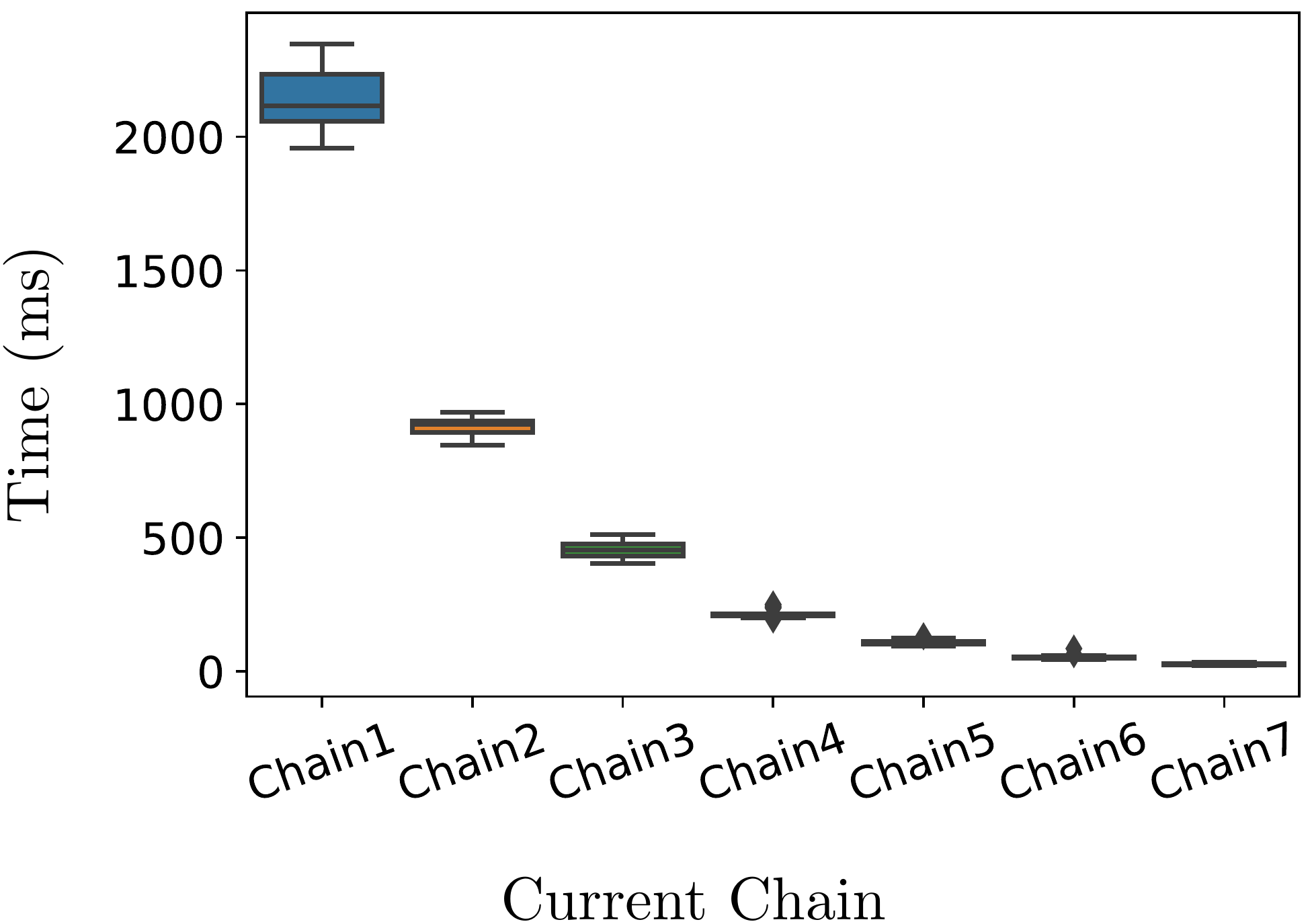}
		\caption{\textbf{Verification time overheads. }
			Time taken by our robust Bayesian verification framework to execute the online quantitative verification and reconfiguration step over 10 independent runs when the robot attempts to clean the indicated chain.}
		\label{fig:verification}
	\end{figure}
	

	\section{Discussion \& Conclusions}
	
	Unlike 
	single-point estimators of Markov model parameters \cite{epifani_model_2009,filieri_formal_2012,calinescu_adaptive_2014,filieri_lightweight_2015}, our Bayesian framework provides interval estimates that capture the inherent uncertainty of these parameters, enabling the robust quantitative verification of systems such as autonomous robots. Through its ability to exploit prior knowledge, the framework differs fundamentally from, and is superior to, a recently introduced approach to synthesising intervals for unknown transition parameters based on the frequentist theory of simultaneous confidence intervals.~\cite{calinescu_2016_formal,10.1007/978-3-662-49674-9_32,calinescu2017rodes} Furthermore, instead of applying the same estimator to all Markov model transition parameters like existing approaches, our framework is the first to handle parameters corresponding to singular and regular events differently. This is an essential distinction, especially for the former type of parameter, for which the absence of observations violates a key premise of existing estimators. Our BIPP estimator avoids this invalid premise, and computes two-sided bounded estimates for singular CTMC transition rates---a considerable extension of our preliminary work to devise one-sided bounded estimates for the singular transition probabilities of discrete-time Markov chains.~\cite{zhao_probabilistic_2019}
	
	The proposed Bayesian framework is underpinned by the theoretical foundations of imprecise probabilities \cite{walter_imprecision_2009,walter_bayesian_2017} and Conservative Bayesian Inference (CBI), \cite{bishop_toward_2011,strigini_software_2013,zhao_assessing_2020} integrated with recent advances in the verification of interval CTMCs.~\cite{ceska_prism_psy_2016} In particular, our BIPP theorems for singular events extend CBI significantly in several ways. First, BIPP operates in the continuous domain for a Poisson process, while previous CBI theorems are applicable to Bernoulli processes in the discrete domain. As such, BIPP enables the runtime quantitative verification of interval CTMCs, and thus the analysis of important properties that are not captured by discrete-time Markov models. Second, CBI is one-side (upper) bounded, and therefore 
	only supports the analysis of undesirable singular events (e.g., catastrophic failures). In contrast, BIPP provides two-sided bounded estimates, therefore also enabling the analysis of ``positive'' singular events (e.g., the completion of difficult one-off tasks). Finally, BIPP can operate with any  \textit{arbitrary} number of confidence bounds as priors, which greatly increases the flexibility of exploiting different types of prior knowledge. 
	
	As illustrated by its application to an AUV infrastructure maintenance mission, our robust quantitative verification framework removes the need for 
	%
	%
	precise prior beliefs, which are typically unavailable in many real-world verification tasks that require Bayesian inference. Instead, the framework enables the exploitation of Bayesian combinations of partial or imperfect prior knowledge, which it uses to derive informed estimation errors (i.e., intervals) for the predicted model parameters. Combined with existing techniques for obtaining this prior knowledge, e.g., the Delphi method and its variants~\cite{ishikawa1993max} or reference class forecasting~\cite{flyvbjerg2008curbing}, the framework increases the trustworthiness of Bayesian inference in highly uncertain scenarios such as those encountered in the verification of autonomous robots.
	
	Based on recent survey papers {\cite{araujo_testing_2023,luckcuck_formal_2019,gleirscher_new_2019}} that provide in-depth discussions on the challenges and opportunities in the field of autonomous robot verification, it has become evident that a common taxonomy emerges, primarily revolving around two key dimensions. The first dimension centres on the specification of properties under verification, which includes various types of temporal logic languages.~{\cite{araujo_testing_2023}} The second dimension pertains to how system behaviours are modeled/structured. In this regard, formal models such as Belief Desire Intention, Petri Nets, and finite state machines, along with their diverse extensions, have emerged as popular approaches to capturing the intricate dynamics of autonomous systems. Our approach falls within the category of methods utilising CSL and CTMCs for the verification of robots. However, unlike the existing methods from this category~{\cite{gerasimou_undersea_2017,younes_numerical_2006}}, we introduced treatments of the model parameters uncertainty via robust Bayesian learning methods, and integrated them with recent research on interval CMTC model checking.
	
	Another important approach for verifying the behaviour of autonomous agents under uncertainty uses hidden Markov models (HMMs).~{\cite{zhang2005logic,hernandez_marimba_2015,WEI2002129}} HMM-based verification supports the analysis of stochastic systems whose true state is not observable, and can only be estimated (with aleatoric uncertainty given by a predefined probability distribution) through monitoring a separate process whose observable state depends on the unknown state of the system. In contrast, our verification framework supports the analysis of autonomous agents whose true state is observable but for which the rates of transition between these known states are affected by epistemic uncertainty and need to be learnt from system observations (as shown in Figure~\ref{fig:integration}). As such, HMM-based verification and our robust verification framework differ substantially by tackling different types of autonomous agent uncertainty. Because autonomous agents may be affected by both types of uncertainty, the complementarity of the two verification approaches can actually be leveraged by using our BIPP and IPSP Bayesian estimators in conjunction with HMM-based verification, i.e., to learn the transition rates associated with continuous-time HMMs that model the behaviour of an autonomous agent. Nevertheless, achieving this integration will first require the development of generic continuous-time HMM verification techniques since, to the best of our knowledge, only verification techniques and tools for the verification of discrete-time HMMs are currently available.

	Although our method demonstrates promising potential, it is not without limitations. One limitation is scalability---as the complexity of the robot's behaviour and the environment grow, the number of unknown parameters to be estimated at runtime may increase, leading to increased computational overheads for our Bayesian estimators. Additionally, the method requires a certain level of expertise to construct the underlying CTMC model structure. This demands understanding both the robot's dynamics and the environment in order to model them as a CTMC, making the approach less accessible to those without specialised knowledge. Last but not least, a challenge inherent to all Bayesian methods involves the acquisition of appropriate priors. While our robust Bayesian estimators mitigate this issue by eliminating the need for complete and precise prior knowledge, establishing the required partial and vague priors can still pose challenges. These limitations suggest important areas for future work.
	
	\section{Methods}

	\medskip
	\hypertarget{bippProof}{}
	\noindent
	\subsection{BIPP estimator proofs.} To prove Theorem~\ref{theorem_BIPP_bounds}, we require the following preliminary results.
	
	\begin{lemma}
		\label{lemma_concave}
		If $l(\cdot)$ is the likelihood function defined in \eqref{eq_likelihood_f}, then $g:(0,\infty)\rightarrow \mathbb{R}$, $g(w)=w \cdot l^{-1}(w)$ is a concave function.
	\end{lemma}
	\begin{proof}
		Since $g(w)=w\cdot \left(-\frac{\ln w}{t}\right)$ and $t>0$, the second derivative of $g$ satisfies 
		\begin{align}
		\frac{d^2 g}{dw^2}=\frac{d}{dw}\left[-\frac{\ln w}{t}-\frac{1}{t}\right]=-\frac{1}{wt}<0.
		\end{align}
		Thus, $g(w)$ is concave.
	\end{proof}

	\begin{proposition} 
		\label{prop_BIPP_m_point_for_upper_bound}
		With the notation from Theorem~\ref{theorem_BIPP_bounds}, there exist $m$ values $\lambda_1\!\in\!(\epsilon_0,\epsilon_1]$, $\lambda_2\!\in\!(\epsilon_1,\epsilon_2]$, \ldots, $\lambda_m\!\in\!(\epsilon_{m-1},\epsilon_m]$ such that $\sup\; {\mathcal S}_\lambda$ is the posterior estimate~(\ref{eq_posterior_mean_lambda}) obtained by using as prior the $m$-point discrete distribution with probability mass $f(\lambda_i)=\mathit{Pr}(\lambda=\lambda_i)=\theta_i$ for $i=1,2,\ldots,m$. 
	\end{proposition}
	\begin{proof}
		\label{proof_theorem_BIPP_m_point_for_upper_bound}
		Since $f(\lambda)=0$ for $\lambda\notin [\epsilon_0,\epsilon_m]$, the Lebesgue-Stieltjes integration from the objective function \eqref{eq_posterior_mean_lambda} can be rewritten as:
		\begin{equation}
		\E(\Lambda \mid \text{data})
		=\frac{\sum\limits_{i=1}^{m}\int_{\epsilon_{i-1}}^{\epsilon_i}\lambda l(\lambda)  f(\lambda)\diff \lambda}{\sum\limits_{i=1}^{m}\int_{\epsilon_{i-1}}^{\epsilon_i} l(\lambda)  f(\lambda)\diff \lambda}
		\label{eq_OB_in_proof}
		\end{equation}
		The first mean value theorem for integrals (e.g.\ \cite[p.~249]{zwillinger_table_2015}) ensures that, for every $i=1,2,\ldots,m$,
		there are points $\lambda_i, \lambda'_{i} \in [\epsilon_{i-1},\epsilon_{i}]$ such that:
		\begin{align}
		\int_{\epsilon_{i-1}}^{\epsilon_i} l(\lambda)f(\lambda) \diff \lambda &=l(\lambda_i)\int_{\epsilon_{i-1}}^{\epsilon_i} f(\lambda) \diff \lambda= l(\lambda_i) \theta_i
		\label{eq_mean_val_theory1}
		\\
		\int_{\epsilon_{i-1}}^{\epsilon_i} \lambda l(\lambda)f(\lambda) \diff \lambda &=\lambda'_{i}l(\lambda'_{i})\int_{\epsilon_{i-1}}^{\epsilon_i} f(\lambda) \diff \lambda=\lambda'_{i} l(\lambda'_{i}) \theta_i
		\label{eq_mean_val_theory2}
		\end{align}
		or, after simple algebraic manipulations of the previous results, 
		\begin{align}
		l(\lambda_i)&=\E[l(\Lambda) \mid \epsilon_{i-1}\leq \Lambda \leq \epsilon_i ]
		\\
		\lambda'_{i} l(\lambda'_{i})&=\E[\Lambda \cdot l(\Lambda) \mid \epsilon_{i-1}\leq \Lambda \leq \epsilon_i ]
		\label{eq_temp_mark}
		\end{align}
		
		\noindent
		Using the shorthand notation $w=l(\lambda)$ for the likelihood function \eqref{eq_likelihood_f} (hence $w>0$), we define $g:(0,\infty)\rightarrow \mathbb{R}$, $g(w)=w \cdot l^{-1}(w)$. According to Lemma~\ref{lemma_concave}, $g(\cdot)$ is a concave function, and thus we have:
		\begin{align}
		\lambda'_{i} l(\lambda'_{i}) 
		& = \E[\Lambda \cdot l(\Lambda) \mid \epsilon_{i-1}\leq \Lambda \leq \epsilon_i ] \nonumber
		\\
		&=\E[ W \cdot l^{-1}(W) \mid \epsilon_{i-1} \leq l^{-1}(W) \leq \epsilon_i] \nonumber
		\\
		&=\E[g(W) \mid \epsilon_{i-1} \leq l^{-1}(W) \leq \epsilon_i] \nonumber
		\\
		&\leq g\left(\E[ W \mid\epsilon_{i-1} \leq l^{-1}(W) \leq \epsilon_i]\right) \label{eq_Jensen_inequality}
		\\
		&= \E[ W\! \mid \epsilon_{i-1} \leq l^{-1}(W) \leq \epsilon_i] \cdot \nonumber\\
		& \hspace*{15mm} l^{-1}\!\left( \E[ W \!\mid\epsilon_{i-1} \leq l^{-1}(W) \leq \epsilon_i] \right)
		\nonumber
		\\
		&=\E[ l(\Lambda) \mid \epsilon_{i-1} \leq \Lambda \leq \epsilon_i] \cdot l^{-1}\left( \E[l(\Lambda) \mid \epsilon_{i-1} \leq \Lambda \leq \epsilon_i] \right) \nonumber
		\\
		&=l(\lambda_i) \cdot l^{-1}\left( l(\lambda_i)\right) \nonumber
		\\
		&=  \lambda_i \cdot l(\lambda_i),
		\label{eq_inequality}
		\end{align}
		where the inequality step \eqref{eq_Jensen_inequality} is obtained by applying Jensen's inequality.\cite{jensen_sur_1906,
			bishop_toward_2011}
		
		We can now use 
		\eqref{eq_mean_val_theory1}, \eqref{eq_mean_val_theory2} and \eqref{eq_inequality} to establish an upper bound for the objective function~\eqref{eq_OB_in_proof}: 
		
		\begin{align}
		\E(\Lambda \mid \text{data})
		&=\frac{\sum\limits_{i=1}^{m} \lambda'_{i}l(\lambda'_{i})\theta_i}
		{\sum\limits_{i=1}^{m} l(\lambda_{i})\theta_i} \leq \frac{\sum\limits_{i=1}^{m} \lambda_{i}l(\lambda_{i})\theta_i}
		{\sum\limits_{i=1}^{m} l(\lambda_{i})\theta_i}
		\label{eq_reachable_bound}
		\end{align}
		This upper bound 
		is attained 
		by selecting an $m$-point discrete distribution $f_u(\lambda)$ with probability mass $\theta_i$ at $\lambda=\lambda_i$, for $i=1,2,\ldots,m$ (since substituting $f(\cdot)$ from \eqref{eq_OB_in_proof} with this $f_u(\cdot)$ yields the rhs result of \eqref{eq_reachable_bound}). 
		As such, maximising this bound 
		reduces to an optimisation problem in the $m$-dimensional space of $(\lambda_1,\lambda_2,\ldots,\lambda_m)\in(\epsilon_0,\epsilon_1]\times(\epsilon_1,\epsilon_2]\times\cdots\times(\epsilon_{m-1},\epsilon_m]$.
		This optimisation problem can be solved numerically, 
		yielding a supremum (rather than a maximum) for $\mathcal{S}_\lambda$ in the case when the optimised prior distribution has points located at $\lambda_{i}=\epsilon_{i-1}$ for  $i=1,2,\ldots,m$.
	\end{proof}
	
	\begin{proposition}
		\label{prop_BIPP_m_point_for_lower_bound}
		With the notation from Theorem~\ref{theorem_BIPP_bounds}, there exist $m$ values $x_1,x_2,\ldots,x_m\in[0,1]$ such that $\inf \mathcal{S}_\lambda$ is the posterior estimate~(\ref{eq_posterior_mean_lambda}) obtained by using as prior the $(m+1)$-point discrete distribution with probability mass $f(\epsilon_0)=\mathit{Pr}(\lambda=\epsilon_0)=x_1\theta_1$,  $f(\epsilon_i)=\mathit{Pr}(\lambda=\epsilon_i)=(1-x_{i})\theta_{i}+x_{i+1}\theta_{i+1}$ for $1\leq i<m$, and $f(\epsilon_m)=\mathit{Pr}(\lambda=\epsilon_m)=(1-x_m)\theta_m$.
	\end{proposition}
	\begin{proof}
		We reuse the reasoning steps from Proposition~\ref{prop_BIPP_m_point_for_upper_bound} up to  inequality \eqref{eq_Jensen_inequality}, which we replace with the following alternative inequality derived from the Converse Jensen's Inequality \cite{lah_converse_1973,bakula_converse_2012} and the fact that $g(w)$ is a concave function (cf. Lemma \ref{lemma_concave}):
		\begin{align}
		\textrm{\hspace*{-2.7mm}}\lambda'_{i} l(\lambda'_{i}) 
		&=\E[g(W) \mid \epsilon_{i-1} \leq l^{-1}(W) \leq \epsilon_i] \nonumber
		\\
		&\geq \frac{l(\epsilon_{i-1})-\E[W \mid \epsilon_{i-1} \leq l^{-1}(W) \leq \epsilon_i ]}{l(\epsilon_{i-1})-l(\epsilon_i)}g(l(\epsilon_i))  \nonumber\\
		&\quad+\frac{\E[W \mid \epsilon_{i-1} \leq l^{-1}(W) \leq \epsilon_i]-l(\epsilon_i)}{l(\epsilon_{i-1})-l(\epsilon_i)}g(l(\epsilon_{i-1})) \nonumber
		\\
		&=\frac{l(\epsilon_{i-1})-l(\lambda_{i})}{l(\epsilon_{i-1})-l(\epsilon_i)}\epsilon_i l(\epsilon_i)+\frac{l(\lambda_{i})-l(\epsilon_i)}{l(\epsilon_{i-1})-l(\epsilon_i)}\epsilon_{i-1}l(\epsilon_{i-1})
		\label{eq_converse_inequality}
		\end{align}
		We can now establish a lower bound for~\eqref{eq_OB_in_proof}:

		\begin{align}
		\E(\Lambda \mid \text{data})
		&=\frac{\sum\limits_{i=1}^{m} \lambda'_{i}l(\lambda'_{i})\theta_i}
		{\sum\limits_{i=1}^{m} l(\lambda_{i})\theta_i} \nonumber
		\\
		&\hspace*{-12mm}
		\geq \frac{\sum\limits_{i=1}^{m} \left(\frac{l(\epsilon_{i-1})-l(\lambda_{i})}{l(\epsilon_{i-1})-l(\epsilon_i)}\epsilon_i l(\epsilon_i)+\frac{l(\lambda_{i})-l(\epsilon_i)}{l(\epsilon_{i-1})-l(\epsilon_i)}\epsilon_{i-1}l(\epsilon_{i-1})\right)\theta_i}
		{\sum\limits_{i=1}^{m} l(\lambda_{i})\theta_i}
		\\
		&=\frac{\sum\limits_{i=1}^{m} \left[\epsilon_i l(\epsilon_i)(1-x_i)\theta_i+\epsilon_{i-1}l(\epsilon_{i-1})x_i\theta_i\right]}
		{\sum\limits_{i=1}^{m} [l(\epsilon_{i})(1-x_i)\theta_i+l(\epsilon_{i-1})x_i\theta_i]}
		\label{eq_reachable_lower_bound}
		\end{align}
		where $x_i$ is defined as:
		\begin{equation}
		x_i=\frac{l(\lambda_{i})-l(\epsilon_i)}{l(\epsilon_{i-1})-l(\epsilon_i)}
		\label{eq_ri_def}
		\end{equation}
		
		The result \eqref{eq_reachable_lower_bound} is essentially in the same form as the result obtained by using a $2m$-point distribution in which, for each interval $[\epsilon_{i-1},\epsilon_{i}]$, there are two points located at $\lambda=\epsilon_{i-1}$ and $\lambda=\epsilon_{i}$ and the probability mass associated with these points is $x_i\theta_i$ and $(1-x_i)\theta_i$ respectively. Intuitively, $x_i$ is the ratio of splitting the probability mass $\theta_i$ between the two points since, according to~\eqref{eq_ri_def},  $x_i \in [0,1]$.
		
		Furthermore, the points on the boundaries of two successive intervals are overlapping, which effectively reduces the number of points from $2m$ to $m+1$. Expanding~\eqref{eq_reachable_lower_bound} yields an $(m+1)$-point discrete distribution $f_l(\lambda)$ with probability mass $f_l(\epsilon_0)=x_1\theta_1$,  $f_l(\epsilon_i)=(1-x_{i})\theta_{i}+x_{i+1}\theta_{i+1}$ for $1\leq i<m$ and $f_l(\epsilon_m)=(1-x_m)\theta_m$. As such, minimising~\eqref{eq_reachable_lower_bound} reduces to an $m$-dimensional optimisation problem in $x_1,x_2,\ldots,x_m$, which can be solved numerically given other model parameters. Finally, since~\eqref{eq_prior_constraints} requires that $\epsilon_{i-1} < \lambda_i \leq \epsilon_i$, we have $0\leq x_i <1$, and thus the posterior estimate is an infimum (rather than a minimum) of $\mathcal{S}_\lambda$ when the solution of the optimisation problem corresponds to a combination of $x_1,x_2,\ldots,x_m$ values that includes one or more values of $1$.	
	\end{proof}
	
	\noindent We can now prove the main theoretical result from Section~1.2. In the supplementary material, we use this result to prove Corollaries~\ref{corollary_closed_form_BIPP_3} and~\ref{corollary_closed_form_BIPP_2}.
	
	\renewenvironment{proof}{\medskip\noindent{\bfseries Proof of Theorem~\ref{theorem_BIPP_bounds}.}}{\qed\smallskip}
	\begin{proof}
		Propositions~\ref{prop_BIPP_m_point_for_upper_bound} and~\ref{prop_BIPP_m_point_for_lower_bound} imply that the set of posterior estimates $\lambda$ over all priors that satisfy the constraints~\eqref{eq_prior_constraints} has:
		\begin{enumerate}
			\item the infinum $\lambda_l$ from~\eqref{eq_reachable_BIPP_lower_bound}, obtained by using the prior $f(\lambda)$ from Proposition~\ref{prop_BIPP_m_point_for_lower_bound} in~\eqref{eq_posterior_mean_lambda};
			\item the supremum $\lambda_u$ from~\eqref{eq_reachable_BIPP_upper_bound}, obtained by using the prior $f(\lambda)$ from Proposition~\ref{prop_BIPP_m_point_for_upper_bound} in \eqref{eq_posterior_mean_lambda}.
		\end{enumerate}
		
		\vspace*{-2mm}
	\end{proof}
	\renewenvironment{proof}{\medskip\noindent{Proof.}}{\qed\smallskip}

	\medskip
	\hypertarget{ipspProof}{}
	\noindent
	\subsection{IPSP estimator proofs.} A formal proof for the results from~\eqref{eq_post_lower_bound_rij} and~\eqref{eq_post_upper_bound_rij} is provided below.
	
	\renewenvironment{proof}{\medskip\noindent{\bfseries Proof of Theorem~\ref{theorem_CTMC_IPSP}.}}{\qed\smallskip}
	\begin{proof}
		To find the extrema for the posterior rate $\lambda^{(t)}$, we first differentiate~\eqref{eq_posterior_rate_gamma} with respect to $\lambda^{(0)}$: 
		\[
		\frac{d}{d \lambda^{(0)}} \left(\lambda^{(t)}\right) = \frac{t^{(0)}}{t+t^{(0)}}.
		\]
		As $t^{(0)}>0$ and $t>0$, this derivative is always positive, so \begin{equation}
		\underline{\lambda}^{(t)}=\min_{t^{(0)}\in[\underline{t}^{(0)},\overline{t}^{(0)}]} \frac{t^{(0)}\underline{\lambda}^{(0)}+n}{t^{(0)}+t}
		\label{eq:th2-step1a}
		\end{equation}
		and 
		\begin{equation}
		\overline{\lambda}^{(t)}=\max_{t^{(0)}\in[\underline{t}^{(0)},\overline{t}^{(0)}]} \frac{t^{(0)}\overline{\lambda}^{(0)}+n}{t^{(0)}+t}.
		\label{eq:th2-step1b}
		\end{equation}
		We now differentiate the quantity that needs to be minimised in~\eqref{eq:th2-step1a} with respect to $t^{(0)}$:
		\[
		\frac{d}{d t^{(0)}} \left(\frac{t^{(0)}\underline{\lambda}^{(0)}+n}{t^{(0)}+t}\right) = \frac{\underline{\lambda}^{(0)}(t^{(0)}+t)-(t^{(0)}\underline{\lambda}^{(0)}+n)\cdot 1}{(t^{(0)}+t)^2} = \frac{\underline{\lambda}^{(0)}t-n}{(t^{(0)}+t)^2},
		\]
		As this derivative is non-positive for $\underline{\lambda}^{(0)}\in\left(0,\frac{n}{t}\right]$ and positive for $\underline{\lambda}^{(0)}>\frac{n}{t}$, the minimum from~\eqref{eq:th2-step1a} is attained for $t^{0}=\overline{t}^{(0)}$ in the former case, and for $t^{0}=\underline{t}^{(0)}$ in the latter case, which yields the result from~\eqref{eq_post_lower_bound_rij}.
		Similarly, the derivative of the quantity to maximise in~\eqref{eq:th2-step1b}, i.e., 
		\[
		\frac{d}{d t^{(0)}} \left(\frac{t^{(0)}\overline{\lambda}^{(0)}+n}{t^{(0)}+t}\right) = \frac{\overline{\lambda}^{(0)}t-n}{(t^{(0)}+t)^2},
		\]
		is non-positive for $\overline{\lambda}^{(0)}\in\left(0,\frac{n}{t}\right]$ and positive for $\overline{\lambda}^{(0)}>\frac{n}{t}$, so the maximum from~\eqref{eq:th2-step1b} is attained for $t^{0}=\underline{t}^{(0)}$ in the former case, and for $t^{0}=\overline{t}^{(0)}$ in the latter case, which yields the result from~\eqref{eq_post_upper_bound_rij} and completes the proof.
	\end{proof}
	\renewenvironment{proof}{\medskip\noindent{Proof.}}{\qed\smallskip}

	
	\medskip
	\noindent
	\subsection{BIPP estimator evaluation.} 
	Fig.~\ref{fig:bippResults} shows the results of experiments we carried out to evaluate the BIPP estimator in scenarios with $m=3$ (Figs.~\ref{fig:BIPPResultsA}--\ref{fig:BIPPResultsC}) and $m=2$ (Fig.~\ref{fig:BIPPResultsD}) confidence bounds by varying the characteristics of the partial prior knowledge. 
	For $m=3$, the upper bound computed by the estimator exhibits a three-stage behaviour as the time over which no singular event occurs increases. These stages correspond to the three $\lambda_u$ regions from \eqref{eq:closed-form-upper-bound}. They start with a steep $\lambda_u$ decrease for $t < \frac{1}{\epsilon_2}$ in stage~1, followed by a slower $\lambda_u$ decreasing trend for $\frac{1}{\epsilon_2} \leq t \leq \frac{1}{\epsilon_1}$ in stage~2, and approaching the asymptotic value $\frac{\epsilon_1(\theta_1+\theta_2)}{\theta_1}$ as the mission progresses through stage~3.  
	Similarly, the lower bound $\lambda_l$ demonstrates a two-stage behaviour, as expected given its two-part definition~\eqref{eq:closed-form-lower-bound}, 
	with the overall value approaching 0 as the mission continues and no singular event modelled by this estimator (e.g., a catastrophic failure) occurs.

	Fig.~\ref{fig:BIPPResultsA} shows the behaviour of the estimator for different $\theta_1$ values and fixed $\theta_2$, $\epsilon_1$ and $\epsilon_2$ values. 
	For  higher $\theta_1$ values, more probability mass is allocated to the confidence bound  $(\epsilon_0, \epsilon_1]$, yielding a steeper decrease in the upper bound $\lambda_u$ and a lower $\lambda_u$ value at the end of the mission.  
	The lower bound $\lambda_l$ presents limited variability across the different $\theta_1$ values, becoming almost constant and close to $0$ as $\theta_1$ increases.

	A similar decreasing pattern is observed in Fig.~\ref{fig:BIPPResultsB}, which depicts the results of experiments with $\theta_1, \epsilon_1$ and $\epsilon_{2}$ fixed, and $\theta_2$ variable.
	The upper bound $\lambda_u$ in the long-term is larger for higher $\theta_2$ values, resulting in a wider posterior estimate bound as $\lambda_u$ converges towards its theoretical asymptotic value.

	Allocating the same probability mass to the confidence bounds, i.e., $\theta_1 = \theta_2 = 0.3$ and changing the prior knowledge bounds $\epsilon_1$ and $\epsilon_2$ affects greatly the behaviour of the BIPP estimator (Fig.~\ref{fig:BIPPResultsC}). 
	When $\epsilon_1$ and $\epsilon_2$ have relatively high values compared to the duration of the mission (e.g., see the first three plots in Fig~\ref{fig:BIPPResultsC}), the upper bound $\lambda_u$ of the BIPP estimator rapidly converges to its asymptotic value, leaving no room for subsequent improvement as the mission progresses. 
	Similarly, the earlier the triggering point for switching between the two parts of the lower bound $\lambda_l$ calculation~\eqref{eq:closed-form-lower-bound}, the earlier $\lambda_l$ reaches a plateau close to 0.

	Finally, Fig.~\ref{fig:BIPPResultsD} shows experimental results for the special scenario comprising only $m=2$ confidence bounds. In this scenario, replacing $\theta_2=0$ in~\eqref{eq:closed-form-lower-bound} as required by Corollary~\ref{corollary_closed_form_BIPP_2} gives a constant lower bound $\lambda_l = 0$ irrespective of the other BIPP estimator parameters. 
	As expected, the upper bound $\lambda_u$ demonstrates a twofold behaviour, featuring a rapid decrease until $t=\frac{1}{\epsilon1}$, followed by a steady state behaviour where $\lambda_u = \frac{\epsilon_1}{\theta_1}$. 

	\begin{figure*}[t!]
		\centering
		\begin{subfigure}[!ht]{\textwidth}
			\centering
			\includegraphics[width=\textwidth]{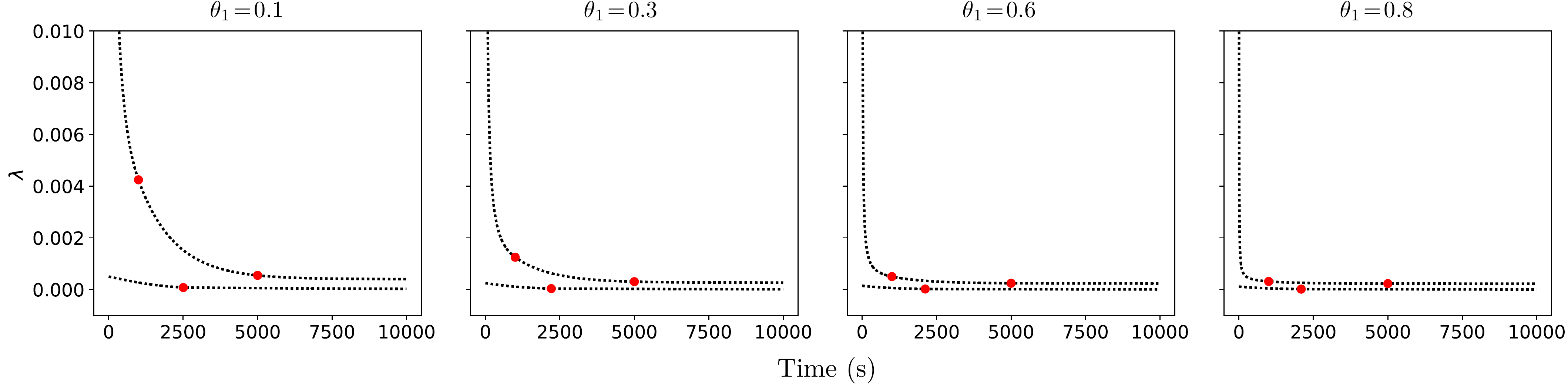}
			\vspace*{-5mm}
			\caption{BIPP estimator for $m=3$, $\theta_1  \in \{0.1, 0.3, 0.6, 0.8\}$, $\theta_2=0.1$, $\epsilon_1=\frac{1}{5000}$, $\epsilon_2=\frac{1}{1000}$
			}
			\label{fig:BIPPResultsA}
		\end{subfigure}
		\hfill
		\vspace*{4mm}
		\begin{subfigure}[!ht]{\textwidth}
			\centering
			\includegraphics[width=\textwidth]{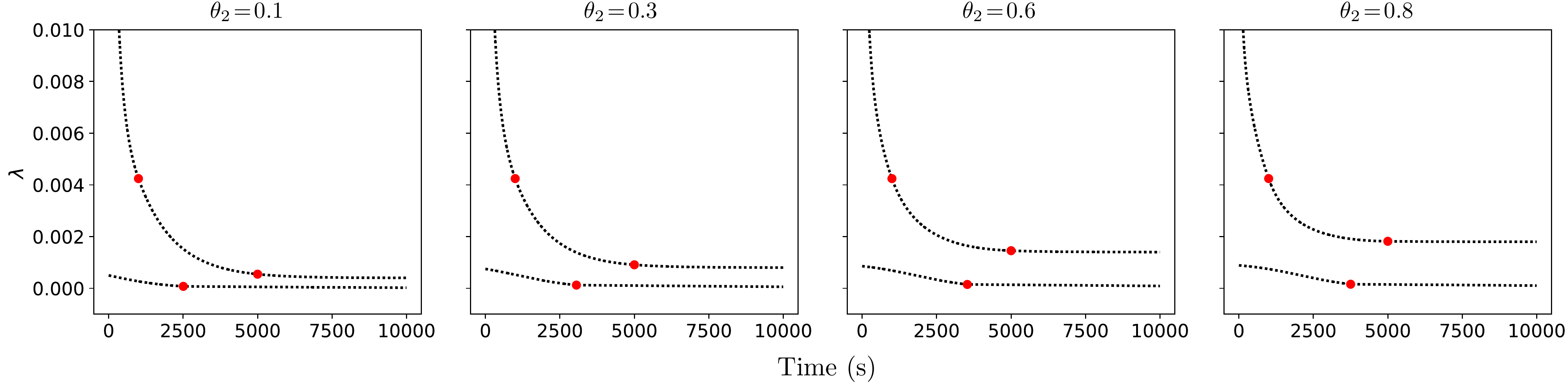}
			\vspace*{-5mm}
			\caption{BIPP estimator for $m=3$, $\theta_1=0.1$, $\theta_2 \in \{0.1, 0.3, 0.6, 0.8\}$, $\epsilon_1=\frac{1}{5000}$, $\epsilon_2=\frac{1}{1000}$
			}
			\label{fig:BIPPResultsB}
		\end{subfigure}
		\hfill
		\vspace*{4mm}
		\begin{subfigure}[!ht]{\textwidth}
			\centering
			\includegraphics[width=\textwidth]{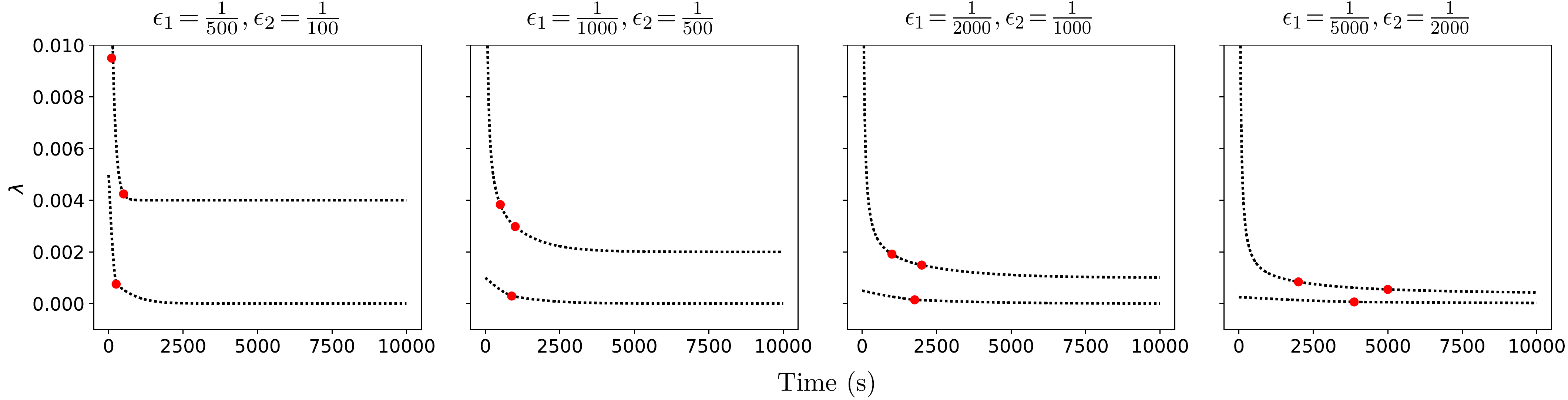}
			\vspace*{-5mm}
			\caption{BIPP estimator for $m=3$, $\theta_1=0.3$, $\theta_2=0.3$, $\left(\epsilon_1,\epsilon_2\right)  \in \left\{\left(\frac{1}{500},\frac{1}{100}\right),\left(\frac{1}{1000},\frac{1}{500}\right),\left(\frac{1}{2000},\frac{1}{1000}\right), \left(\frac{1}{5000},\frac{1}{2000}\right)\right\}$
			}
			\label{fig:BIPPResultsC}
		\end{subfigure}
		\hfill
		\vspace*{4mm}
		\begin{subfigure}[!ht]{\textwidth}
			\centering
			\includegraphics[width=\textwidth]{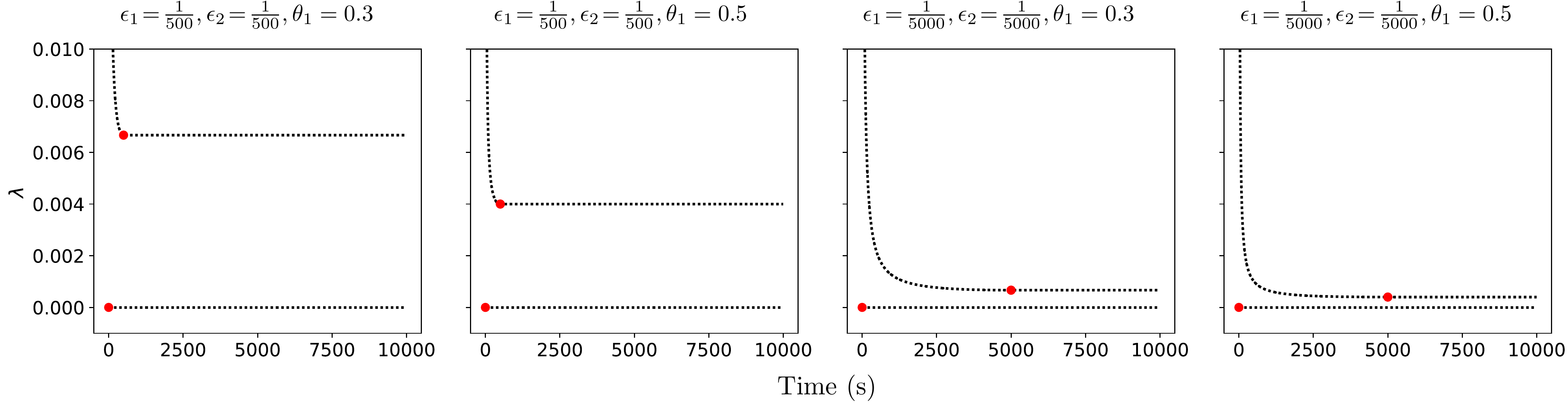}
			\vspace*{-5mm}
			\caption{BIPP estimator for $m=2$, $\theta_1 \in \{0.3,0.5\}$, $\theta_2=0$, $\left(\epsilon_1,\epsilon_2\right)  \in \left\{\left(\frac{1}{500},\frac{1}{500}\right),\left(\frac{1}{5000},\frac{1}{5000}\right)\right\}$
			}
			\label{fig:BIPPResultsD}
		\end{subfigure}
		\caption{\textbf{Experimental analysis of the Bayesian inference using partial priors (BIPP) estimator.} Systematic experimental analysis of the BIPP estimator showing the bounds $\lambda_l$ and $\lambda_u$ of the posterior estimates for the occurrence probability of singular events for the duration of a mission. 
			Each plot shows the effect of different partial prior knowledge encoded in \eqref{eq_prior_constraints} on the calculation of the lower \eqref{eq_reachable_BIPP_lower_bound} and upper \eqref{eq_reachable_BIPP_upper_bound} posterior estimate bounds. The red circles indicate the time points when the different formulae for the lower and upper bounds in \eqref{eq:closed-form-lower-bound} and \eqref{eq:closed-form-upper-bound}, respectively, become active. 
		}
		\label{fig:bippResults}
	\end{figure*}

	\begin{figure*}[t!]
		\centering
		\begin{subfigure}[!ht]{\textwidth}
			\centering
			\includegraphics[width=\textwidth]{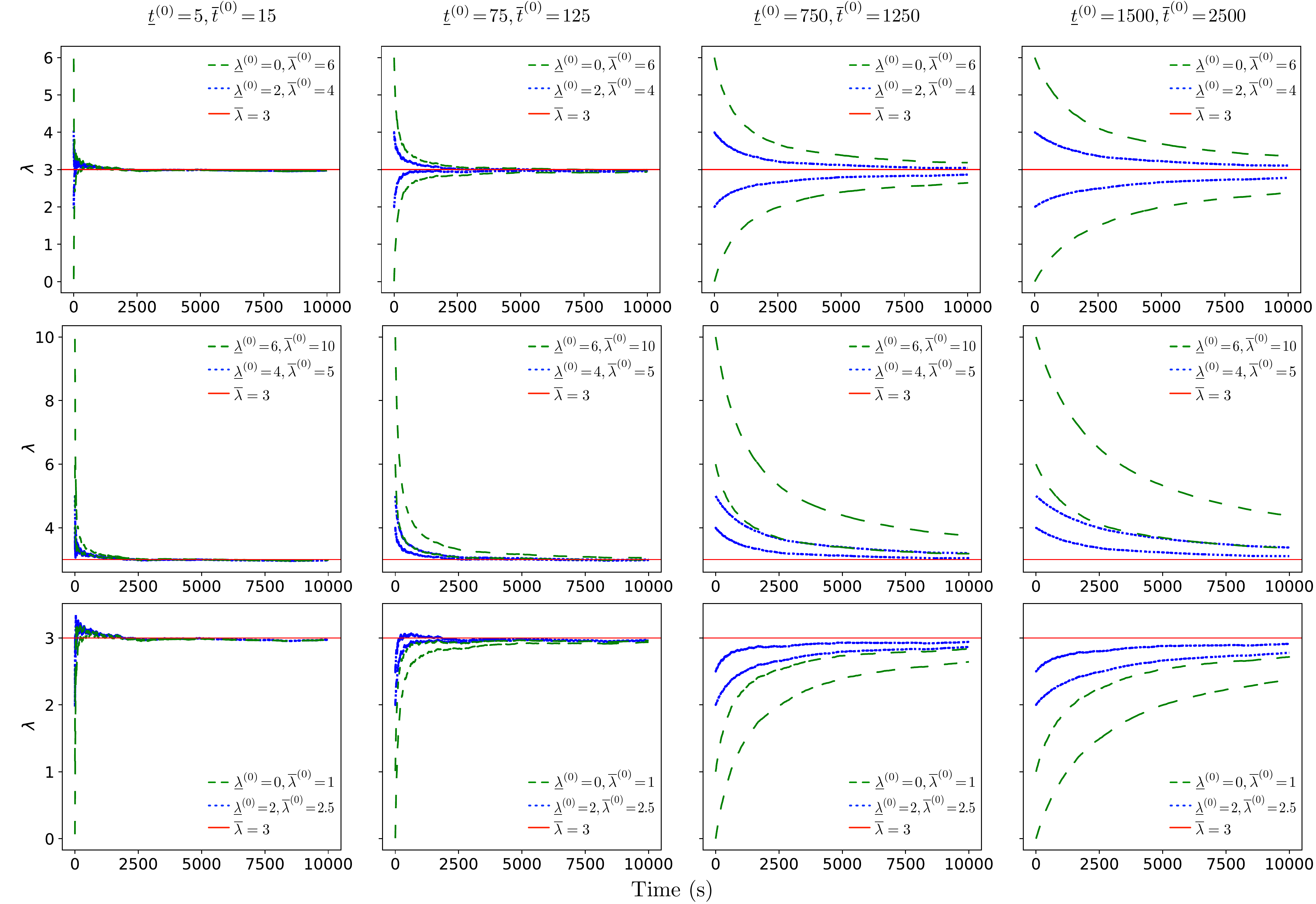}
			\caption{
				IPSP estimator results showing the impact of different sets of priors $[\underline{t}^{(0)},\overline{t}\textrm{}^{(0)}]$ and $[\underline{\lambda}^{(0)},\overline{\lambda}\textrm{}^{(0)}]$. In each plot, the blue dotted line (\textcolor{blue}{$\cdots\cdot$}) and green dashed line ({\textcolor{darkgreen}{$---$}}) show the posterior estimation bounds $\underline{\lambda}^{(t)}$ and $\overline{\lambda}\textrm{}^{(t)}$ for narrow and wide $[\underline{\lambda}^{(0)},\overline{\lambda}\textrm{}^{(0)}]$ intervals, respectively. 
				Each column of plots corresponds to assigning different strength to the prior knowledge, ranging from uninformative (leftmost column) to strong belief (rightmost column). 
				The first row shows scenarios in which the actual rate $\overline{\lambda}=3$ belongs to the prior knowledge interval $[\underline{\lambda}^{(0)},\overline{\lambda}\textrm{}^{(0)}]$. 
				In the second and third rows, the prior intervals overestimate and underestimate $\overline{\lambda}$, respectively. }
			\label{fig:ipspResultsA}
		\end{subfigure}
		\hfill
		\vspace*{4mm}
		\begin{subfigure}[!ht]{\textwidth}
			\centering
			\includegraphics[width=\textwidth]{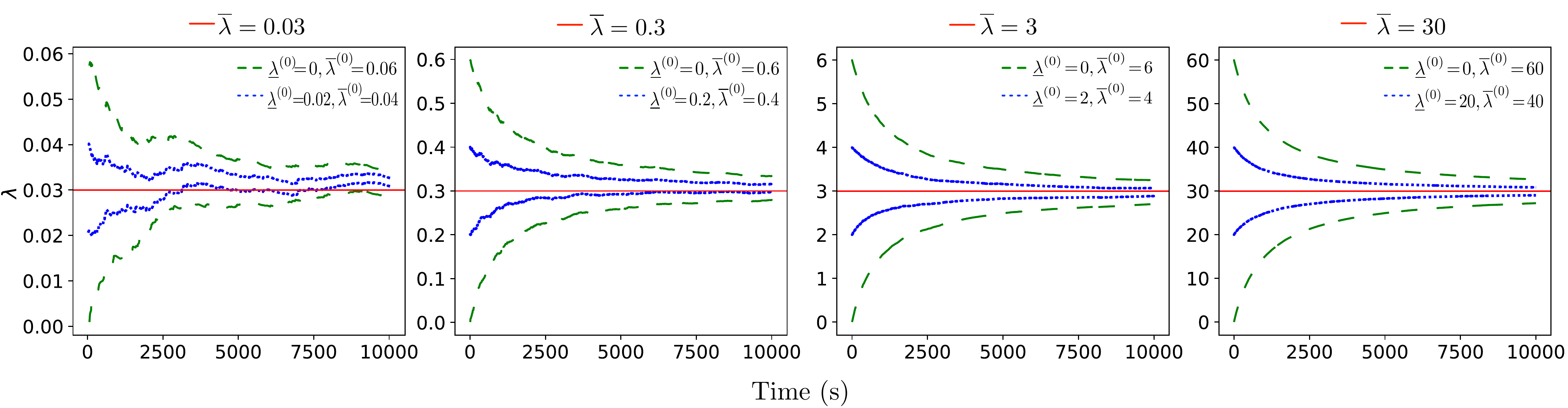}
			\caption{IPSP estimator results illustrating the behaviour of IPSP across different actual rate values 
				$\overline{\lambda} \in \{0.03, 0.3, 3, 30\}$. The experiments were carried out for $[\underline{t}^{(0)},\overline{t}\textrm{}^{(0)}] = [1000, 1000]$ and included both narrow and wide $[\underline{\lambda}^{(0)},\overline{\lambda}\textrm{}^{(0)}]$ intervals, which are shown in blue dotted lines (\textcolor{blue}{$\cdots\cdot$}) and green dashed lines (\textcolor{darkgreen}{$---$}), respectively. In all experiments, the unknown actual rate $\overline{\lambda}$ was in the prior interval $[\underline{\lambda}^{(0)},\overline{\lambda}\textrm{}^{(0)}]$.
			}
			\label{fig:ipspResultsb}
		\end{subfigure}
		\caption{\textbf{Experimental analysis of the Bayesian inference using imprecise probability with sets of priors (IPSP) estimator. } Systematic experimental analysis of the IPSP estimator showing the bounded posterior estimators for regular events. }
		\label{fig:ipspResults}
	\end{figure*}

	\medskip
	\noindent
	\subsection{IPSP estimator evaluation.} 
	Fig.~\ref{fig:ipspResults} shows the results of experiments we performed to analyse the behaviour of the IPSP estimator in scenarios with varying ranges for the prior knowledge
	$[\underline{t}^{(0)},\overline{t}\textrm{}^{(0)}]$ and $[\underline{\lambda}^{(0)},\overline{\lambda}\textrm{}^{(0)}]$.
	A general observation is that the posterior rate intervals $[\underline{\lambda}^{(t)},\overline{\lambda}^{(t)}]$ become narrower as the mission progresses, irrespective of the level of trust assigned to the prior knowledge, i.e., across all columns of plots (which correspond to different $[\underline{t}^{(0)},\overline{t}\textrm{}^{(0)}]$ intervals) from Fig.~\ref{fig:ipspResultsA}.
	Nevertheless, this trust level affects how the estimator incorporates observations into the calculation of the posterior interval. 
	When the trust in the prior knowledge is weak (in the plots from the leftmost columns of Fig.~\ref{fig:ipspResultsA}), the impact of the prior knowledge on the posterior estimation is low, and the IPSP calculation is heavily influenced by the observations, resulting in a narrow interval. 
	In contrast, when the trust in the prior knowledge is stronger (in the plots from the rightmost columns), the contribution of the prior knowledge to the posterior estimation becomes higher, and the IPSP estimator produces a wider interval.

	In the experiments from the first row of plots in Fig.~\ref{fig:ipspResultsA}, the (unknown) actual rate $\overline{\lambda}=3$ belongs to the prior knowledge interval 
	$[\underline{\lambda}^{(0)},\overline{\lambda}\textrm{}^{(0)}]$. As a result, the posterior rate interval $[\underline{\lambda}^{(t)},\overline{\lambda}\textrm{}^{(t)}]$  progressively becomes narrower, approximating $\overline{\lambda}$ with high accuracy.  
	As expected, the narrower prior knowledge (blue dotted line) produces a narrower posterior rate interval than the wider and more conservative prior knowledge (green dashed line). 
	
	When the prior knowledge interval $[\underline{\lambda}^{(0)},\overline{\lambda}\textrm{}^{(0)}]$ overestimates or underestimates the actual rate $\overline{\lambda}$ (second and third rows of plots from Fig.~\ref{fig:ipspResultsA}, respectively), the ability of IPSP to adapt its estimations to reflect the observations heavily depends on the characteristics of the sets of priors. 
	For example, if the width of the prior knowledge $[\underline{\lambda}^{(0)},\overline{\lambda}\textrm{}^{(0)}]$  is close to  $\overline{\lambda}$ and $t^{(0)}\ll t$, then IPSP more easily approaches $\overline{\lambda}$, as shown by the narrow prior knowledge (blue dotted line) in {Fig~\ref{fig:ipspResultsA}} for $[\underline{t}^{(0)},\overline{t}\textrm{}^{(0)}] \in \{[5,15], [75,125], [750,1250]\}$.
	In contrast,  wider narrow prior knowledge (green dashed line) combined with higher levels of trust in the prior, e.g., $[\underline{t}^{(0)},\overline{t}\textrm{}^{(0)}] \in \{[1500,2500]\}$, entails that more observations are needed for the posterior rate to approach the actual rate $\overline{\lambda}$. When the actual rate is, in addition, nonstationary, change-point detection methods can be employed to identify these changes~\cite{epifani_change_point_2010,zhao2020interval} and recalibrate the IPSP estimator. 
	Finally, Fig.~\ref{fig:ipspResultsb} shows the behaviour of IPSP for different actual rate $\overline{\lambda}$ values, i.e., $\overline{\lambda} \in \{0.03, 0.3, 3, 30\}$. 
	As $\overline{\lambda}$ increases, more observations are produced in the same time period, resulting in a smoother and narrower posterior bound estimate.

	\bigskip
	\noindent
	\textbf{Data availability}
	\ \\ \noindent
	The data supporting the RBV findings and a video of the robotic mission in simulation are available at~\url{https://gerasimou.github.io/RBV}.
	
	\bigskip
	\noindent
	\textbf{Code availability}
	\ \\ \noindent
	All code developed in this project is freely available at~\url{http://github.com/gerasimou/RBV}.

	

	
	
	%

	\balance
	\bibliographystyle{unsrt}
	\bibliography{ref.bib}

	\bigskip
	\noindent
	\textbf{Acknowledgments}\\
	This project has received funding from the ORCA-Hub PRF project `COVE', the Assuring Autonomy International Programme, the UKRI project EP/V026747/1 `Trustworthy Autonomous Systems Node in Resilience', and the European Union’s Horizon 2020 project SESAME (grant agreement No 101017258).

	\bigskip
	\noindent
	\textbf{Author contributions}\\
	\noindent
	X.Z: Conceptualisation, Data curation, Formal analysis, Funding acquisition, Investigation, Methodology, Project administration, Software, Supervision, Validation, Writing - original draft, Writing - review and editing;
	S.G.: Conceptualisation, Data curation, Funding acquisition, Investigation, Methodology, Project administration, Software, Supervision, Validation, Visualisation, Writing - original draft, Writing - review and editing;
	R.C.: Conceptualisation, Formal analysis, Funding acquisition, Methodology, Project administration, Supervision, Visualisation, Writing - original draft, Writing - review and editing;
	C.I.: Data curation, Investigation,  Validation, Visualisation, Writing - original draft, Writing - review and editing;
	V.R.: Conceptualisation,  Funding acquisition, Methodology, Project administration, Supervision, Writing - review and editing;
	D.F.:  Conceptualisation,  Funding acquisition, Methodology, Project administration, Supervision, Writing - review and editing.
	
	\bigskip
	\noindent
	\textbf{Competing interests}\\
	The authors declare no competing interests.

	\bigskip
	\noindent
	\textbf{Additional information}\\
	\textbf{Supplementary information} The online version contains supplementary material available at~\url{https://arxiv.org/abs/2303.08476}.

	\bigskip
	\noindent
	\textbf{Correspondence} should be addressed to Xingyu Zhao or Simos Gerasimou.



\section*{Supplementary material}

# Supplementary Material: Bayesian Learning for the Robust Verification of Autonomous Robots

Xingyu Zhao[1], Simos Gerasimou[2], Radu Calinescu[2,3], Calum Imrie[2,3], Valentin Robu[4,5], and David Flynn[6]

[1]Warwick Manufacturing Group, University of Warwick, Coventry, UK.
[2]Department of Computer Science, University of York, York, UK
[3]Assuring Autonomy International Programme, University of York, York, UK.
[4]Intelligent and Autonomous Systems Group, Centrum Wiskunde & Informatica, Amsterdam, NL.
[5]Electrical Engineering Department, Eindhoven University of Technology, Eindhoven, NL
[6]James Watt School of Engineering, University of Glasgow, Glasgow, UK.

## Supplementary Notes 1: Introduction

This supplementary material document includes:

- The proofs to **Corollary 1** and **Corollary 2** from Section 2.3 of the main paper.
- Details of the experimental settings for the offshore infrastructure maintenance case study from Section 3 of the main paper.

## Supplementary Methods 1: Corollary Proofs

**Corollary 1.** *When $m=3$, the bounds (6) and (7) in **Theorem 1** of the main paper satisfy:*

$$\lambda_l \geq \begin{cases} \frac{\epsilon_1 l(\epsilon_1)\theta_2}{\theta_1 + l(\epsilon_1)\theta_2}, & \text{if } \frac{\theta_2(\epsilon_1-\epsilon_2)}{\theta_1} > \frac{\epsilon_2 l(\epsilon_2) - \epsilon_1 l(\epsilon_1)}{l(\epsilon_1)l(\epsilon_2)} \\ \frac{\epsilon_2 l(\epsilon_2)\theta_2}{\theta_1 + l(\epsilon_2)\theta_2}, & \text{otherwise} \end{cases} \tag{S1}$$

*and*

$$\lambda_u < \begin{cases} \frac{\epsilon_1 l(\epsilon_1)\theta_1 + \epsilon_2 l(\epsilon_2)\theta_2 + \frac{1}{t}l(\frac{1}{t})(1-\theta_1-\theta_2)}{l(\epsilon_1)\theta_1}, & \text{if } t < \frac{1}{\epsilon_2} \\ \frac{\epsilon_1 l(\epsilon_1)\theta_1 + \frac{1}{t}l(\frac{1}{t})\theta_2 + \epsilon_2 l(\epsilon_2)(1-\theta_1-\theta_2)}{l(\epsilon_1)\theta_1}, & \text{if } \frac{1}{\epsilon_2} \leq t \leq \frac{1}{\epsilon_1} \\ \frac{\epsilon_1 l(\epsilon_1)(\theta_1+\theta_2) + \epsilon_2 l(\epsilon_2)(1-\theta_1-\theta_2)}{l(\epsilon_1)\theta_1}, & \text{otherwise} \end{cases} \tag{S2}$$

*Proof.* When $m=3$, Eq. (7) of **Theorem 1** says, there is a supremum $\lambda_{u,m=3}$:

$$\lambda_{u,m=3} = \max_{\{0 \leq \lambda_1 \leq \epsilon_1 < \lambda_2 \leq \epsilon_2 < \lambda_3 < +\infty\}} \frac{\lambda_1 l(\lambda_1)\theta_1 + \lambda_2 l(\lambda_2)\theta_2 + \lambda_3 l(\lambda_3)(1-\theta_1-\theta_2)}{l(\lambda_1)\theta_1 + l(\lambda_2)\theta_2 + l(\lambda_3)(1-\theta_1-\theta_2)} \tag{S3}$$

Similarly, Eq. (6) of **Theorem 1** shows, when $m=3$, there is an infimum $\lambda_{l,m=3}$:

$$\lambda_{l,m=3} = \min_{\{0 \leq x_i \leq 1, \forall i \in [1..3]\}} \frac{\sum_{i=1..3}[\epsilon_i l(\epsilon_i)(1-x_i)\theta_i + \epsilon_{i-1}l(\epsilon_{i-1})x_i\theta_i]}{\sum_{i=1..3}[l(\epsilon_i)(1-x_i)\theta_i + l(\epsilon_{i-1})x_i\theta_i]} \tag{S4}$$

where $\epsilon_0 = 0$ and $\epsilon_3 = +\infty$ (and thus $l(\epsilon_0)=1$, $\lim_{\epsilon_3 \to +\infty} l(\epsilon_3) = 0$ and $\lim_{\epsilon_3 \to +\infty} \epsilon_3 l(\epsilon_3) = 0$).

**First, we prove the result of** (S2). By taking the partial derivative of the objective function in (S3) w.r.t. $\lambda_1$, we know the derivative is always positive, irrespective of the values $\lambda_2$ and $\lambda_3$ take in their respective ranges, as shown below (note $0 \leq \lambda_1 \leq \epsilon_1 < \lambda_2 \leq \epsilon_2 < \lambda_3 < +\infty$):

$$\frac{\partial \frac{\lambda_1 l(\lambda_1)\theta_1 + \lambda_2 l(\lambda_2)\theta_2 + \lambda_3 l(\lambda_3)(1-\theta_1-\theta_2)}{l(\lambda_1)\theta_1 + l(\lambda_2)\theta_2 + l(\lambda_3)(1-\theta_1-\theta_2)}}{\partial \lambda_1} = \frac{e^{-\lambda_1 t}\theta_1 \left[e^{-\lambda_1 t}\theta_1 + e^{-\lambda_2 t}\theta_2(1-(\lambda_1-\lambda_2)t) + e^{-\lambda_3 t}(1-\theta_1-\theta_2)(1-(\lambda_1-\lambda_3)t)\right]}{(e^{-\lambda_1 t}\theta_1 + e^{-\lambda_2 t}\theta_2 + e^{-\lambda_3 t}(1-\theta_1-\theta_2))^2} > 0 \tag{S5}$$

This implies that the maximum point lies in the hyperplane of $\lambda_1 = \epsilon_1$. Thus, we substitute $\lambda_1 = \epsilon_1$ into (S3) and reduce the problem to:

$$\lambda_{u,m=3} = \max_{\{\epsilon_1 < \lambda_2 \leq \epsilon_2 < \lambda_3 < +\infty\}} \frac{\epsilon_1 l(\epsilon_1)\theta_1 + \lambda_2 l(\lambda_2)\theta_2 + \lambda_3 l(\lambda_3)(1-\theta_1-\theta_2)}{l(\epsilon_1)\theta_1 + l(\lambda_2)\theta_2 + l(\lambda_3)(1-\theta_1-\theta_2)} \tag{S6}$$

$$< \max_{\{\epsilon_1 < \lambda_2 \leq \epsilon_2 < \lambda_3 < +\infty\}} \frac{\epsilon_1 l(\epsilon_1)\theta_1 + \lambda_2 l(\lambda_2)\theta_2 + \lambda_3 l(\lambda_3)(1-\theta_1-\theta_2)}{l(\epsilon_1)\theta_1} \tag{S7}$$

$$\leq \begin{cases} \frac{\epsilon_1 l(\epsilon_1)\theta_1 + \epsilon_2 l(\epsilon_2)\theta_2 + \frac{1}{t}l(\frac{1}{t})(1-\theta_1-\theta_2)}{l(\epsilon_1)\theta_1} & t < \frac{1}{\epsilon_2} \\ \frac{\epsilon_1 l(\epsilon_1)\theta_1 + \frac{1}{t}l(\frac{1}{t})\theta_2 + \epsilon_2 l(\epsilon_2)(1-\theta_1-\theta_2)}{l(\epsilon_1)\theta_1} & \frac{1}{\epsilon_2} \leq t \leq \frac{1}{\epsilon_1} \\ \frac{\epsilon_1 l(\epsilon_1)(\theta_1+\theta_2) + \epsilon_2 l(\epsilon_2)(1-\theta_1-\theta_2)}{l(\epsilon_1)\theta_1} & t > \frac{1}{\epsilon_1} \end{cases} \tag{S8}$$

where the last step is due to the fact that the function $xl(x)$ is unimodal over $[0,1]$ with a maximum point at $x = \frac{1}{t}$. Thus, the last step says:

- When $t < \frac{1}{\epsilon_2}$ (i.e. $\epsilon_2 < \frac{1}{t}$): the function $\lambda_3 l(\lambda_3)$ can reach its maximum at $\lambda_3 = \frac{1}{t}$ in the range $(\epsilon_2, +\infty)$; While, since $\lambda_2 \in (\epsilon_1, \epsilon_2]$, the function $\lambda_2 l(\lambda_2)$ cannot reach $\lambda_2 = \frac{1}{t}$, so we set $\lambda_2 = \epsilon_2$ to maximise the objective function.

- When $\frac{1}{\epsilon_2} \leq t \leq \frac{1}{\epsilon_1}$ (i.e. $\epsilon_1 \leq \frac{1}{t} \leq \epsilon_2$): the function $\lambda_2 l(\lambda_2)$ can attain its maximum at $\lambda_2 = \frac{1}{t}$ in the range $(\epsilon_1, \epsilon_2]$; While, since $\lambda_3 \in (\epsilon_2, +\infty]$, the function $\lambda_3 l(\lambda_3)$ cannot reach $\lambda_3 = \frac{1}{t}$, so we set $\lambda_3 = \epsilon_2$ to maximise the objective function.

- When $t > \frac{1}{\epsilon_1}$ (i.e. $\frac{1}{t} < \epsilon_1$) both the functions $\lambda_3 l(\lambda_3)$ $\lambda_2 l(\lambda_2)$ take the left endpoints in their range to maximise the objective function, so we set $\lambda_3 = \epsilon_2$ and $\lambda_2 = \epsilon_1$.

Substitute the values of $\lambda_2$ and $\lambda_3$ into the objective function in those three cases, we obtain the results of (S2).

**Now we prove the result of** (S1). If we denote the objective function in (S4) as a fraction $\frac{Nu(x_1,x_2,x_3)}{De(x_1,x_2,x_3)}$, then take its partial derivative w.r.t. $x_3$:

$$\frac{\partial \frac{Nu(x_1,x_2,x_3)}{De(x_1,x_2,x_3)}}{\partial x_3} = \frac{l(\epsilon_2)(1-\theta_1-\theta_2)[((1-x_1)\theta_1 + x_2\theta_2)(\epsilon_2-\epsilon_1)l(\epsilon_1) + \epsilon_2 x_1\theta_1]}{De(x_1,x_2,x_3)^2} > 0 \tag{S9}$$

Thus to minimise the objective function, we set $x_3 = 0$. Then we take its partial derivative w.r.t. $x_1$:

$$\frac{\partial \frac{Nu(x_1,x_2,0)}{De(x_1,x_2,0)}}{\partial x_1} = \frac{-\theta_1[\epsilon_1 l(\epsilon_1) De(x_1,x_2,0) + (1-l(\epsilon_1))Nu(x_1,x_2,0)]}{De(x_1,x_2,0)^2} < 0 \tag{S10}$$

Thus to minimise the objective function, we set $x_1 = 1$. Now we take its partial derivative w.r.t. $x_2$:

$$\frac{\partial \frac{Nu(1,x_2,0)}{De(1,x_2,0)}}{\partial x_2} = \frac{\theta_2[\theta_2(\epsilon_1-\epsilon_2)l(\epsilon_1)l(\epsilon_2) + \theta_1\epsilon_1 l(\epsilon_1) - \theta_1\epsilon_2 l(\epsilon_2)]}{De(1,x_2,0)^2} \tag{S11}$$

whose sign is determined by other model parameters. Thus, we set $x_2 = \mathbf{1}_{\theta_2(\epsilon_1-\epsilon_2)l(\epsilon_1)l(\epsilon_2)+\theta_1\epsilon_1 l(\epsilon_1)-\theta_1\epsilon_2 l(\epsilon_2)<0}$ where $\mathbf{1}_\mathtt{S}$ is an indicator function – it equals 1 when predicate $\mathtt{S}$ is true, and 0 otherwise.

Substitute $x_1 = 1, x_3 = 0$ and $x_2 = \mathbf{1}_{\theta_2(\epsilon_1-\epsilon_2)l(\epsilon_1)l(\epsilon_2)+\theta_1\epsilon_1 l(\epsilon_1)-\theta_1\epsilon_2 l(\epsilon_2)<0}$ into $\frac{Nu(x_1,x_2,x_3)}{De(x_1,x_2,x_3)}$, we obtain two cases in (S1).

□

**Corollary 2.** *The closed-form BIPP bounds for $m = 2$ can be obtained respectively by setting $\epsilon_2 = \epsilon_1$ and $\theta_2 = 0$ in the results* (S1) *and* (S2).

*Proof.* When $m = 2$, Eq. (7) of **Theorem 1** becomes the supremum $\lambda_{u,m=2}$ such that (note, $\theta_2 = 1 - \theta_1$):

$$\lambda_{u,m=2} = \max_{\{0 \leq \lambda_1 \leq \epsilon_1 < \lambda_2 < +\infty\}} \frac{\lambda_1 l(\lambda_1)\theta_1 + \lambda_2 l(\lambda_2)(1 - \theta_1)}{l(\lambda_1)\theta_1 + l(\lambda_2)(1 - \theta_1)} \tag{S12}$$

Similarly, Eq. (6) of **Theorem 1** becomes the infimum $\lambda_{l,m=2}$:

$$\lambda_{l,m=2} = \min_{\{0 \leq x_1 \leq 1, 0 \leq x_2 \leq 1\}} \frac{\epsilon_0 l(\epsilon_0) x_1 \theta_1 + \epsilon_1 l(\epsilon_1)(1 - x_1)\theta_1 + \epsilon_1 l(\epsilon_1) x_2 (1 - \theta_1) + \epsilon_2 l(\epsilon_2)(1 - x_2)(1 - \theta_1)}{l(\epsilon_0) x_1 \theta_1 + l(\epsilon_1)(1 - x_1)\theta_1 + l(\epsilon_1) x_2 (1 - \theta_1) + l(\epsilon_2)(1 - x_2)(1 - \theta_1)} \tag{S13}$$

where $\epsilon_0 = 0$ and $\epsilon_2 = +\infty$.

**First, we prove the bound $\lambda_{u,m=2}$ satisfies:**

$$\lambda_{u,m=2} < \begin{cases} \frac{\epsilon_1 l(\epsilon_1)\theta_1 + \frac{1}{t} l(\frac{1}{t})(1 - \theta_1)}{l(\epsilon_1)\theta_1} & t < \frac{1}{\epsilon_1} \\ \frac{\epsilon_1}{\theta_1} & t \geq \frac{1}{\epsilon_1} \end{cases} \tag{S14}$$

for which we proceed in two steps:

1. We show the optimised point in the two dimensional space of $\lambda_1$ and $\lambda_2$ must lie in the plane of $\lambda_1 = \epsilon_1$.

2. In the plane of $\lambda_1 = \epsilon_1$, a closed-form expression can be derived from the monotonicity analysis of $\lambda_2$.

By taking the partial derivative of the objective function in (S12) w.r.t. $\lambda_1$, we know the derivative is always positive, irrespective of the value take $\lambda_2$ in its respective range, as shown in (S15) below (note, $0 \leq \lambda_1 \leq \epsilon_1 < \lambda_2 < +\infty$):

$$\frac{\partial \frac{\lambda_1 e^{-\lambda_1 t}\theta_1 + \lambda_2 e^{-\lambda_2 t}(1 - \theta_1)}{e^{-\lambda_1 t}\theta_1 + e^{-\lambda_2 t}(1 - \theta_1)}}{\partial \lambda_1} = \frac{e^{-\lambda_1 t}\theta_1 \left[e^{-\lambda_1 t}\theta_1 + e^{-\lambda_2 t}(1 - \theta_1)(1 - (\lambda_1 - \lambda_2)t)\right]}{(e^{-\lambda_1 t}\theta_1 + e^{-\lambda_2 t}(1 - \theta_1))^2} > 0 \tag{S15}$$

This implies that the maximum point lies in the plane of $\lambda_1 = \epsilon_1$. Now we reduce the optimisation problem from a two-dimensional space to the one-dimensional space of $\lambda_2$. Thus, by substituting $\lambda_1 = \epsilon_1$ in to the r.h.s. of (S12), we have:

$$\begin{aligned} \lambda_{u,m=2} &\leq \max_{\{\lambda_2 > \epsilon_1\}} \frac{\epsilon_1 l(\epsilon_1)\theta_1 + \lambda_2 l(\lambda_2)(1 - \theta_1)}{l(\epsilon_1)\theta_1 + l(\lambda_2)(1 - \theta_1)} \\ &< \max_{\{\lambda_2 > \epsilon_1\}} \frac{\epsilon_1 l(\epsilon_1)\theta_1 + \lambda_2 l(\lambda_2)(1 - \theta_1)}{l(\epsilon_1)\theta_1} \\ &< \begin{cases} \frac{\epsilon_1 l(\epsilon_1)\theta_1 + \frac{1}{t} l(\frac{1}{t})(1 - \theta_1)}{l(\epsilon_1)\theta_1} & t < \frac{1}{\epsilon_1} \\ \frac{\epsilon_1}{\theta_1} & t \geq \frac{1}{\epsilon_1} \end{cases} \end{aligned} \tag{S16}$$

where the last step of (S16) is because of the monotonicity analysis of the term $\lambda_2 l(\lambda_2)$ as follows. Depends on the the observable $t$:

- When $\epsilon_1 < \frac{1}{t}$, $\lambda_2 l(\lambda_2)$ attains its maximum at the critical point $\lambda_2 = \frac{1}{t}$, in the range $\lambda_2 > \epsilon_1$. Thus, we substitute $\lambda_2 = \frac{1}{t}$ and obtain the first case in result (S16).

- When $\epsilon_1 \geq \frac{1}{t}$, in the range $\lambda_2 > \epsilon_1$, we know the supremum of $\lambda_2 l(\lambda_2)$ is attained at the boundary point $\lambda_2 = \epsilon_1$ . Thus, we substitute $\lambda_2 = \epsilon_1$ and obtain the second case in result (S16).

**Second, we prove the infimum $\lambda_{l,m=2} = 0$ with the optimal point at $x_1 = 1, x_2 = 0$.** Since $l(0) = 1$, $\lim_{\epsilon_2 \to +\infty} l(\epsilon_2) = 0$ and $\lim_{\epsilon_2 \to +\infty} \epsilon_2 l(\epsilon_2) = 0$, (S13) can be rewritten as:

$$\lambda_{l,m=2} = \min_{\{0 \leq x_1 \leq 1, 0 \leq x_2 \leq 1\}} \frac{\epsilon_1 l(\epsilon_1)(1 - x_1)\theta_1 + \epsilon_1 l(\epsilon_1) x_2 (1 - \theta_1)}{x_1 \theta_1 + l(\epsilon_1)(1 - x_1)\theta_1 + l(\epsilon_1) x_2 (1 - \theta_1)} \tag{S17}$$

The partial derivative of the objective function in (S17) w.r.t. $x_2$ is:

$$\frac{\partial \frac{\epsilon_1 l(\epsilon_1)(1-x_1)\theta_1 + \epsilon_1 l(\epsilon_1) x_2 (1-\theta_1)}{x_1 \theta_1 + l(\epsilon_1)(1-x_1)\theta_1 + l(\epsilon_1) x_2 (1-\theta_1)}}{\partial x_2} = \frac{\epsilon_1 l(\epsilon_1)(1-\theta_1)\theta_1 x_1}{[((x_1 + x_2 - 1)\theta_1 - x_2)l(\epsilon_1) - \theta_1 x_1]^2} > 0 \qquad (S18)$$

Thus we set $x_2 = 0$ in (S17) to reduce the problem to:

$$\lambda_{l,m=2} = \min_{\{0 \leq x_1 \leq 1\}} \frac{\epsilon_1 l(\epsilon_1)(1-x_1)\theta_1}{x_1 \theta_1 + l(\epsilon_1)(1-x_1)\theta_1} \qquad (S19)$$

The partial derivative of the objective function in (S19) w.r.t. $x_1$ is:

$$\frac{\partial \frac{\epsilon_1 l(\epsilon_1)(1-x_1)\theta_1}{x_1 \theta_1 + l(\epsilon_1)(1-x_1)\theta_1}}{\partial x_1} = \frac{-\epsilon_1 l(\epsilon_1)}{[x_1 + (1-x_1)l(\epsilon_1)]^2} < 0 \qquad (S20)$$

Thus we set $x_1 = 1$ in (S19), and obtain $\lambda_{l,m=2} = 0$. Note, the result of 0 is attainable meaning we cannot find a lower bound that bigger than 0 for the given optimisation problem.

**Finally,** substitute $\epsilon_2 = \epsilon_1$ and $\theta_2 = 0$ in the results (S2) and (S1), we obtain the results of (S14) and 0 which are the closed-form BIPP bounds for $m = 2$.

□

# Supplementary Methods 2: Offshore Infrastructure Maintenance Experiments

### Simulation Platform

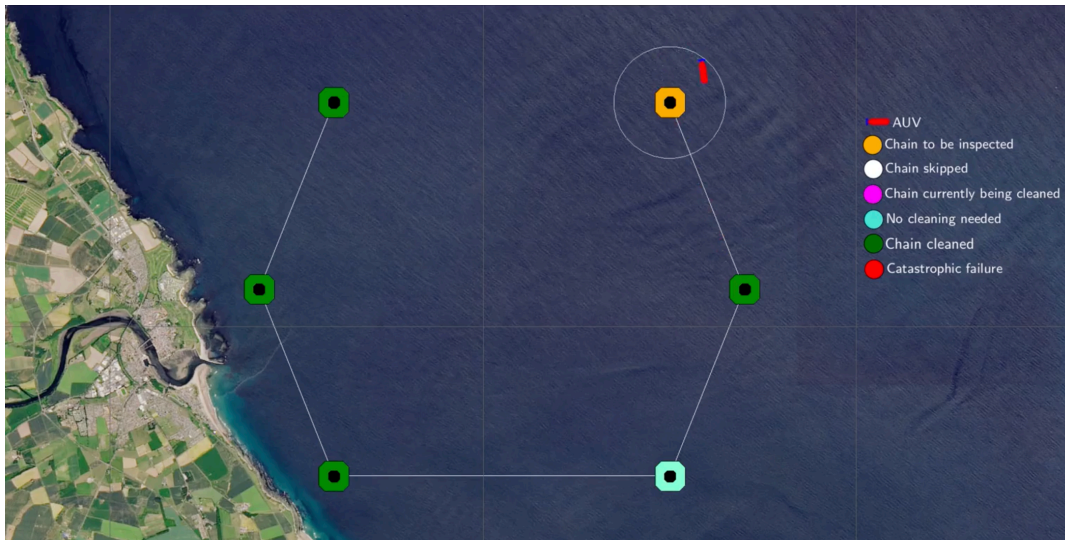

Figure 1: Illustration of our robust Bayesian verification framework for the structural health inspection and cleaning mission using an autonomous underwater vehicle (AUV) at the point when the AUV inspects the final floating chain.

In Section 2 of the main paper, we demonstrate the application of our robust Bayesian verification framework using a case study that involves an autonomous underwater vehicle (AUV) executing a structural health inspection and cleaning mission of the substructure of an offshore wind farm. The offshore wind farm consists of multiple floating wind turbines. Each turbine is a buoyant foundation structure secured to the sea bed with floating chains tethered to anchors. The AUV is deployed to collect data about the condition of the floating chains to enable the post-mission identification of problems that could affect the structural integrity of the chains. Figure 1 shows the AUV during the inspection of the last floating chain.

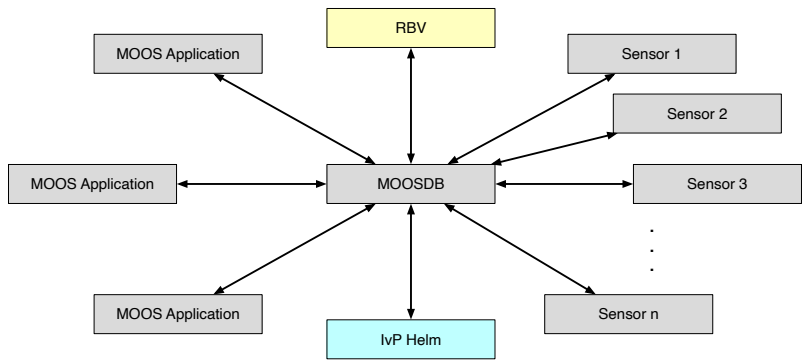

Figure 2: High-level MOOS-IvP architecture with the RBV framework implementation

The AUV-based mission is built on top of the open-source framework MOOS-IvP[1], a widely used platform for the implementation of autonomous applications with AUVs. When used for the execution of oceanic missions, MOOS-IvP is deployed on the payload computer of an AUV, facilitating the decoupling of the vehicle's autonomy from the navigation and control system running on the main AUV computer [1].

An AUV-based system leveraging MOOS-IvP is structured as a community of independent applications running in parallel that communicate via a MOOS database (MOOSDB) using a publish-subscribe architecture. Figure 2 shows the high-level architecture of MOOS-IvP. Applications publish messages in the form of key-value pairs with specified frequencies, sharing information about AUV components that an application monitors. Interested listening applications can use the keys to subscribe to messages and receive a notification when an update of that message becomes available.

The autonomous operation in MOOV-IvP is instrumented through a collection of behaviours, i.e., combinations of boolean logic constraints and piecewise-linear utility functions parametrised, for example, with parameters of the navigation and control system such as heading, speed or depth. During mission execution, the IvP Helm, the decision-making component of MOOS-IvP, periodically collects and reconciles the instantiated behaviours. If multiple behaviours are active simultaneously, the IvP Helm executes Interval Programming (IvP) multi-objective optimisation to determine the optimal action, i.e., an optimal point in the decision space defined by the constraints and utility functions. This optimal action is expressed as a set of key–value pairs and is published to the MOOSDB so that interested (subscribing) applications can receive this update and act upon it.

To realise the AUV-based floating chain inspection and maintenance mission, we extended the MOOS-IvP framework and developed a new MOOS application (called RBV in Figure 2) that implements the overall mission scenario and controls the mission execution. In particular, the RBV application employs the built-in behaviours MOOS-IvP (e.g., waypoint and station keep) to model the AUV mission and leverages the starting and ending condition of these behaviours to instrument the decision-making via the IvP Helm. Furthermore, the RBV application provides several configuration parameters that enable the execution of custom experiments. For instance, users can define the probabilities and rates characterising the behaviour of each chain (i.e., specialising the continuous-time Markov chain – CTMC, model in the main paper), thus, affecting the UAV behaviour. Using a seed as a configuration parameter enables to reduce the non-determinism of the simulator, thus enhancing the reproducibility of the experiments and the robustness of the results obtained.

The open-source RBV source code, the full experimental results, additional information about RBV, including a video of the floating chain inspection and maintenance mission, are available at https://github.com/gerasimou/RBV.

### Experimental Methodology

We evaluated the capabilities of our RBV framework by performing a wide range of experiments that assess both the decision support offered by the framework and its overheads. Accordingly, we instrumented the simulation platform described in Section  with the implemented RBV framework (main paper, Figure 1) and realised the AUV-driven structural health inspection and cleaning mission presented in Section 2 of the main paper. Given the parametric CTMC model of the mission (main paper, Figure 2), we consider as unknown parameters the chain-dependent transition rate for cleaning the $i$-th chain ($r_i^{\text{clean}}$), and the mission-dependent transition rates for causing

---

[1]http://www.moos-ivp.org

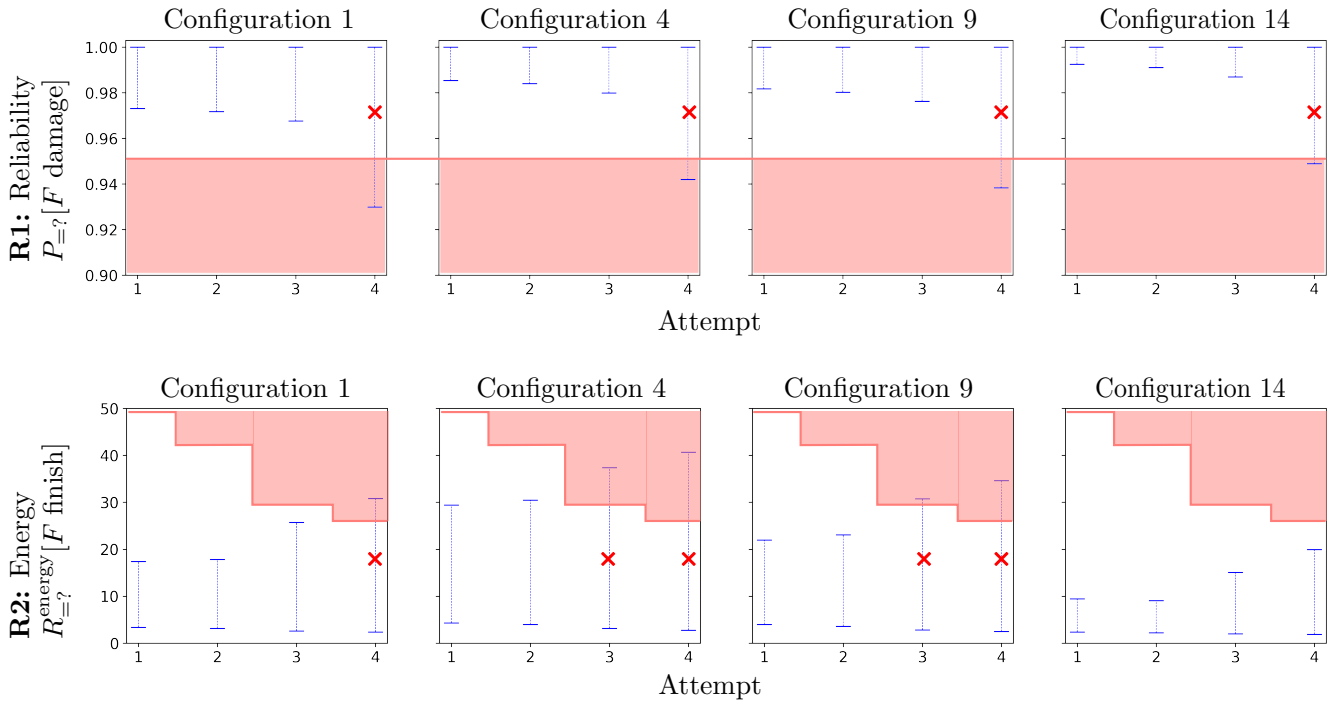

Figure 3: Computed value intervals for the reliability requirement R1, the probability that the AUV will not encounter a catastrophic failure during its mission (top) and energy requirement R2, the expected energy consumption (bottom), over successive attempts for the same AUV configuration. After a failed attempt, each new attempt for the same chain and AUV configuration results in a wider interval for the key system requirements R1 and R2.

catastrophic damage to a floating chain or itself ($r^{\text{damage}}$) and for failing to clean ($r^{\text{fail}}$).[2]

We assemble the interval CTMC model using the BIPP and IPSP estimators to learn these unknown model parameters. In particular, we use the BIPP estimator to quantify the rate values associated with the singular events of cleaning the $i$-th chain ($r_i^{\text{clean}}$) and encountering a catastrophic failure ($r^{\text{damage}}$). The former corresponds to successfully completing a difficult one-off task, and the latter models a major failure. Since the AUV may try multiple times to clean a particular chain, we model the corresponding transition rate ($r^{\text{fail}}$) using the IPSP estimator, which is suitable for events observed regularly during system operation.

## Results

We have already presented how our RBV framework supports the runtime verification of mission-critical autonomous robots for a typical scenario of the AUV-based offshore wind-turbine inspection and maintenance mission (main paper, Figure 3). We also measured the overheads associated with executing the online verification process (main paper, Figure 4). Furthermore, we systematically analysed the operation of both BIPP and IPSP estimators in several scenarios with varying levels of partial prior knowledge (main paper, Figures 5 and 6).

In this section, we present additional results for the end-to-end application of the RBV framework, focusing on the AUV behaviour over multiple failed attempts to clean a specific chain. Figure 3 shows the verification results for requirements R1 – quantifying the probability of the mission completing successfully (top) and R2 – quantifying the expected energy consumption of the AUV (bottom) across successive attempts for the same AUV configuration. In each of these plots and irrespective of the system property measured, the computed value intervals become wider as the number of failed AUV attempts to clean the chain increases. For instance, consider requirement R1 and configuration 1 (shown on the top left in Figure 3), which shows a small increase in the reliability interval for the three initial attempts to clean the chain. Despite the interval becoming wider, the reliability threshold of 0.95 is satisfied; thus, this configuration is feasible and is included in the candidates set for further analysis using requirement R3 – selecting the configuration that maximises the number of chains cleaned. In contrast, the computed reliability interval for the fourth attempt violates the reliability threshold; thus, this configuration is

---

[2]Since the floating chains are spatially located in the same area, we model the failure rate $r^{\text{fail}}$ as a homogeneous parameter affecting all chains of the mission similarly. Nevertheless, our RBV framework can be easily adapted to support modelling an individual transition rate for failing to clean ($r_i^{\text{fail}}$) each $i$-th chain.

infeasible. No valid configuration exists in the fourth attempt, and the AUV decides to skip the chain and move to the next.

A similar pattern of wider value intervals is also observed for the energy consumption property (R2). In this case, the energy threshold decreases for each new attempt as the AUV has consumed energy trying to clean the chain in the previous attempts. Consequently, this requirement is more restrictive and leads to excluding further configurations; see, for instance, the violated energy threshold in attempt 3 for configurations 4 and 9.

The wider intervals over each successive failed attempt correspond to the increased uncertainty concerning the AUV's operation and its capacity to fulfil the mission successfully. The rationale underpinning this behaviour is that since both transition rates $r_i^{\text{clean}}$ and $r^{\text{damage}}$ employ the BIPP estimator, the posterior estimate bounds for both transition rates are wider and converge towards their theoretical asymptotic values (main paper, Section 4.4). However, since the prior knowledge for the $r_i^{\text{clean}}$ rate is higher than the $r^{\text{damage}}$ rate, the posterior bounds for the $r_i^{\text{clean}}$ rate decline much faster than those of the $r^{\text{damage}}$ rate, leading to a more conservative estimate and a wider interval for requirements R1 and R2.

# Supplementary References


\end{document}